\documentclass{article}

\usepackage{microtype}
\usepackage{graphicx}
\usepackage{subfigure}
\usepackage{booktabs} % for professional tables
\usepackage{multirow}
\usepackage{hyperref}

\usepackage[accepted]{icml2020}
\usepackage{bm}
\usepackage{amsmath}
\usepackage{amssymb}
\usepackage{amsthm}
\newtheorem{theorem}{Theorem}[section]

\newtheorem{proposition}[theorem]{Proposition}

\newtheorem{definition}[theorem]{Definition}

\icmltitlerunning{Learning Autoencoders with Relational Regularization}

\begin{document}

\twocolumn[
\icmltitle{Learning Autoencoders with Relational Regularization}



\icmlsetsymbol{equal}{*}

\begin{icmlauthorlist}
\icmlauthor{Hongteng Xu}{equal,in,du}
\icmlauthor{Dixin Luo}{equal,du}
\icmlauthor{Ricardo Henao}{du}
\icmlauthor{Svati Shah}{du}
\icmlauthor{Lawrence Carin}{du}
\end{icmlauthorlist}

\icmlaffiliation{du}{Duke University, Durham, NC, USA}
\icmlaffiliation{in}{Infinia ML Inc., Durham, NC, USA}

\icmlcorrespondingauthor{Hongteng Xu}{hongtengxu313@gmail.com}
\icmlcorrespondingauthor{Dixin Luo}{dixin.luo@duke.edu}

\icmlkeywords{Autoencoder, relational regularization, fused Gromov-Wasserstein}

\vskip 0.3in
]



\printAffiliationsAndNotice{\icmlEqualContribution} % otherwise use the standard text.

\begin{abstract}
A new algorithmic framework is proposed for learning autoencoders of data distributions. 
We minimize the discrepancy between the model and target distributions, with a \emph{relational regularization} on the learnable latent prior. 
This regularization penalizes the fused Gromov-Wasserstein (FGW) distance between the latent prior and its corresponding posterior, allowing one to flexibly learn a structured prior distribution associated with the generative model. 
Moreover, it helps co-training of multiple autoencoders even if they have heterogeneous architectures and incomparable latent spaces. 
We implement the framework with two scalable algorithms, making it applicable for both probabilistic and deterministic autoencoders. 
Our relational regularized autoencoder (RAE) outperforms existing methods, $e.g.$, the variational autoencoder, Wasserstein autoencoder, and their variants, on generating images. 
Additionally, our relational co-training strategy for autoencoders achieves encouraging results in both synthesis and real-world multi-view learning tasks. 
The code is at \url{https://github.com/HongtengXu/Relational-AutoEncoders}.
\end{abstract}

\section{Introduction}
Autoencoders have been used widely in many challenging machine learning tasks for generative modeling, $e.g.$, image~\cite{kingma2013auto,tolstikhin2018wasserstein} and sentence ~\cite{bowman2016generating,wang2019topic} generation.
Typically, an autoencoder assumes that the data in the sample space $\mathcal{X}$ may be mapped to a low-dimensional manifold, which can be represented in a latent space $\mathcal{Z}$.
The autoencoder fits the unknown data distribution $p_x$ via a latent-variable model denoted $p_G$, specified by a prior distribution $p_z$ on latent code $z\in\mathcal{Z}$ and a generative model $G:\mathcal{Z}\mapsto\mathcal{X}$ mapping the latent code to the data $x\in\mathcal{X}$. 
Learning seeks to minimize the discrepancy between $p_x$ and $p_G$. 
According to the choice of the discrepancy, we can derive different autoencoders.
For example, the variational autoencoder~\cite{kingma2013auto} applies the KL-divergence as the discrepancy and learns a probabilistic autoencoder via maximizing the evidence lower bound (ELBO).
The Wasserstein autoencoder (WAE)~\cite{tolstikhin2018wasserstein} minimizes a relaxed form of the Wasserstein distance between $p_x$ and $p_G$, and learns a deterministic autoencoder. 
In general, the objective function approximating the discrepancy consists of a reconstruction loss of observed data and a regularizer penalizing the difference between the prior distribution $p_z$ and the posterior derived by encoded data, $i.e.$, $q_{z|x}$.
Although existing autoencoders have achieved success in many generative tasks, they often suffer from the following two problems.

\textbf{Regularizer misspecification}
Typical autoencoders, like the VAE and WAE, fix the $p_z$ as a normal distribution, which often leads to the problem of over-regularization. 
Moreover, applying such an unstructured prior increases the difficulties in conditional generation tasks.
To avoid oversimplified priors, the Gaussian mixture VAE (GMVAE)~\cite{dilokthanakul2016deep} and the VAE with VampPrior~\cite{tomczak2018vae} characterize their priors as learnable mixture models. 
However, without side information~\cite{wang2019topic}, jointly learning the autoencoder and the prior suffers from a high risk of under-regularization, which is sensitive to the setting of hyperparameters ($e.g.$, the number of mixture components and the initialization of the prior).

\textbf{Co-training of heterogeneous autoencoders}
Solving a single task often relies on the data in different domains ($i.e.$, multi-view data).
For example, predicting the mortality of a patient may require both her clinical record and genetic information. 
In such a situation, we may need to learn multiple autoencoders to extract latent variables as features from different views.
Traditional multi-view learning strategies either assume that the co-trained autoencoders share the same latent distributions~\cite{wang2015deep,ye2016learning}, or assume that there exists an explicit transform between different latent spaces~\cite{wang2016coupled}. 
These assumptions are questionable in practice, as the corresponding autoencoders can have heterogeneous architectures and incomparable latent spaces. 
How to co-train such heterogeneous autoencoders is still an open problem.

To overcome the aforementioned problems, we propose a new \textbf{R}elational regularized \textbf{A}uto\textbf{E}ncoder (RAE).
As illustrated in Figure~\ref{fig:scheme_1}, we formulate the prior $p_z$ as a Gaussian mixture model. 
Differing from existing methods, however, we leverage the Gromov-Wasserstein (GW) distance~\cite{memoli2011gromov} to regularize the structural difference between the prior and the posterior in a relational manner, $i.e.$, comparing the distance between samples from the prior with samples from the posterior, and restricting their difference. 
Considering this relational regularizer allows us to implement the discrepancy between $p_z$ and $q_{z|x}$ as the fused Gromov-Wasserstein (FGW) distance~\cite{vayer2018fused}. 
Besides imposing structural constraints on the prior distribution within a single autoencoder, for multiple autoencoders with different latent spaces ($e.g.$, the 2D and 3D latent spaces shown in Figure~\ref{fig:scheme_2}) we can train them jointly by applying the relational regularizer to their posterior distributions. 

The proposed relational regularizer is applicable for both probabilistic and deterministic autoencoders, corresponding to approximating the FGW distance as hierarchical FGW and sliced FGW, respectively. 
We demonstrate the rationality of these two approximations and analyze their computational complexity. 
Experimental results show that $i$) learning RAEs helps achieve structured prior distributions and also suppresses the under-regularization problem, outperforming related approaches in image-generation tasks; and $ii$) the proposed relational co-training strategy is beneficial for learning heterogeneous autoencoders, which has potential for multi-view learning tasks.

\section{Relational Regularized Autoencoders}
\subsection{Learning mixture models as structured prior}
Following prior work with autoencoders~\cite{tolstikhin2018wasserstein,kolouri2018sliced}, we fit the model distribution $p_G$ by minimizing its Wasserstein distance to the data distribution $p_x$, $i.e.$, $\min D_{\text{w}}(p_x, p_G)$. 
According to Theorem 1 in~\cite{tolstikhin2018wasserstein}, we can relax the Wasserstein distance and formulate the learning problem as follows:
\begin{eqnarray}\label{eq:typical}
\min_{G,Q}~\underbrace{\mathbb{E}_{p_x}\mathbb{E}_{q_{z|x; Q}}[d(x, G(z))]}_{\text{reconstruction loss}} + \underbrace{\gamma D(\mathbb{E}_{p_x}[q_{z|x; Q}],p_z)}_{\text{distance(posterior, prior)}},
\end{eqnarray}
where $G:\mathcal{Z}\mapsto\mathcal{X}$ is the target generative model (decoder); $q_{z|x;Q}$ is the posterior of $z$ given $x$, parameterized by an encoder $Q:\mathcal{X}\mapsto\mathcal{Z}$;
$d$ represents the distance between samples; and 
$D$ is an arbitrary discrepancy between distributions. 
Accordingly, $q_{z;Q}=\mathbb{E}_{p_x}[q_{z|x; Q}]$ is the marginal distribution derived from the posterior.
Parameter $\gamma$ achieves a trade-off between reconstruction loss and the regularizer.

Instead of fixing $p_z$ as a normal distribution, we seek to learn a structured prior associated with the autoencoder:
\begin{eqnarray}\label{eq:ae+prior}
\sideset{}{}\min_{G,Q,p_z\in\mathcal{P}} \mathbb{E}_{p_x}\mathbb{E}_{q_{z|x; Q}}[d(x, G(z))]  + \gamma D(q_{z;Q},p_z).
\end{eqnarray}
where $\mathcal{P}$ is the set of valid prior distributions, which is often assumed as a set of (Gaussian) mixture models~\cite{dilokthanakul2016deep,tomczak2018vae}. 
Learning the structured prior allows one to explore the clustering structure of the data and achieve conditional generation ($i.e.$, sampling latent variables from a single component of the prior and generating samples accordingly). 

\begin{figure}[t]
    \centering
    \subfigure[Proposed RAE]{
    \includegraphics[height=4.4cm]{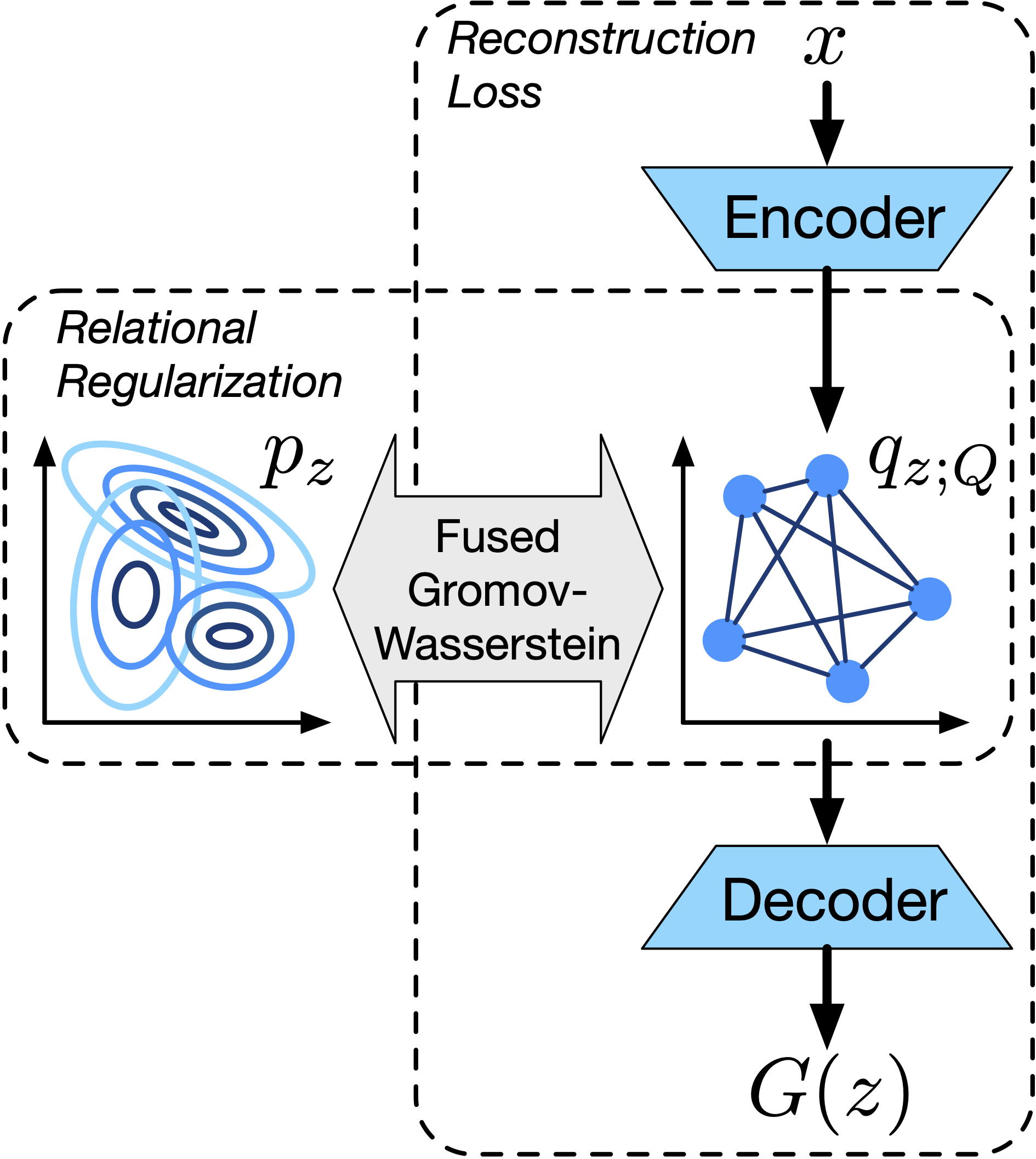}\label{fig:scheme_1}
    }
    \subfigure[Relational Co-training]{
    \includegraphics[height=4.4cm]{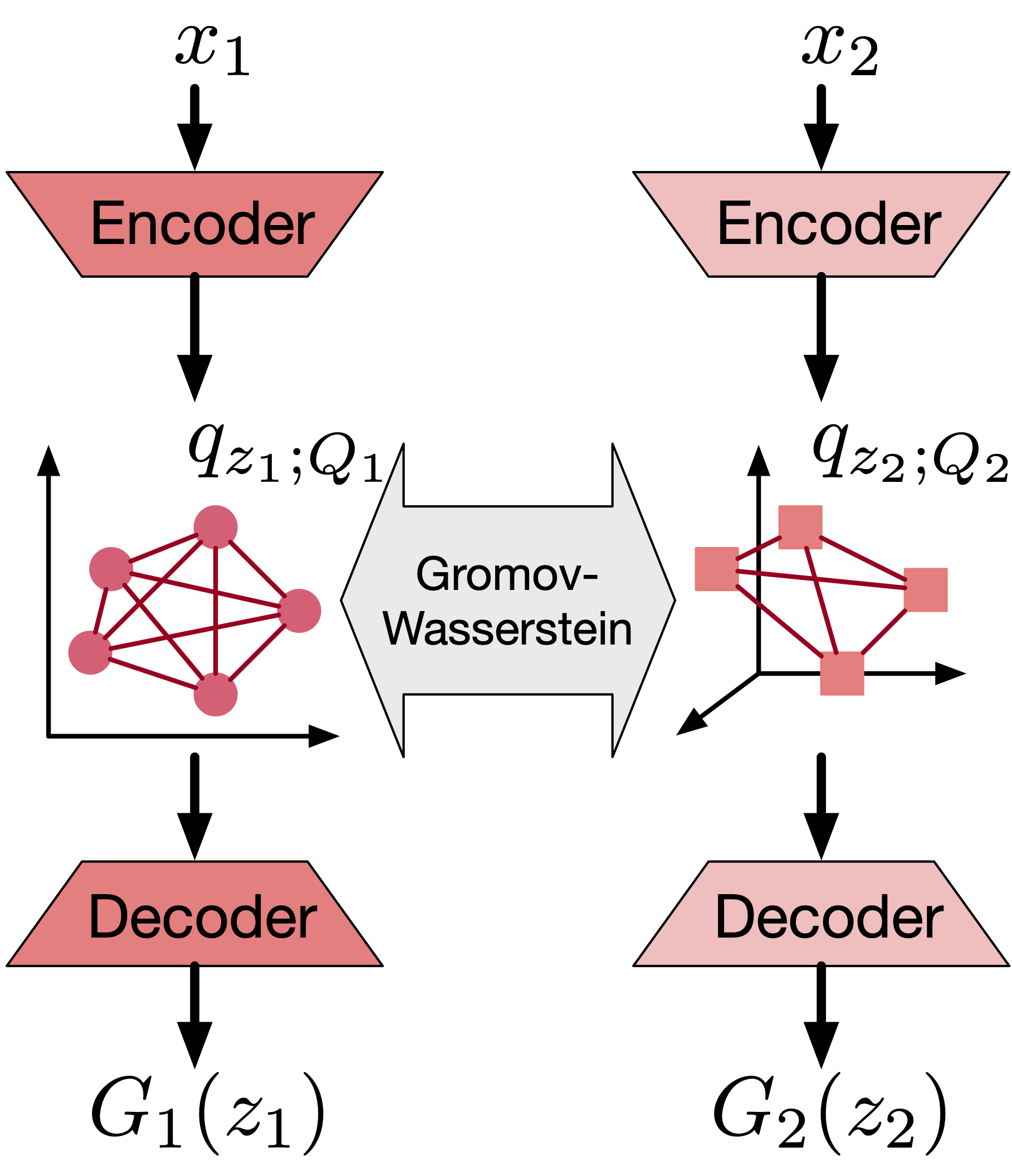}\label{fig:scheme_2}
    }
    \vspace{-10pt}
    \caption{(a) Learning a single autoencoder with relational regularization. (b) Relational co-training of the autoencoders with incomparable latent spaces.}
    \label{fig:scheme}
\end{figure}

\subsection{Relational regularization via Gromov-Wasserstein}
Jointly learning the prior and the autoencoder may lead to under-regularization in the training phase -- it is easy to fit $p_z$ to $q_{z; Q}$ without harm to the reconstruction loss. 
Solving this problem requires introduction of structural constraints when comparing these two distributions, motivating a relational regularized autoencoder (RAE). 
In particular, besides commonly-used regularizers like the KL divergence~\cite{dilokthanakul2016deep} and the Wasserstein distance~\cite{titouan2019sliced}, which achieve \emph{direct} comparisons of the distributions, we consider a \emph{relational} regularizer based on the Gromov-Wasserstein (GW) distance~\cite{memoli2011gromov} in our learning problem:
\begin{eqnarray}\label{eq:ae+prior+gw}
\begin{aligned}
&\sideset{}{_{G,Q,p_z\in\mathcal{P}}}\min~\mathbb{E}_{p_x}\mathbb{E}_{q_{z|x;Q}}[d(x, G(z))] \\
&\quad+ \gamma(\underbrace{(1-\beta) D(q_{z;Q},p_z)}_{\text{direct comparison}} + \underbrace{\beta D_{\text{gw}}(q_{z;Q}, p_z)}_{\text{relational comparison}}).
\end{aligned}
\end{eqnarray}
where $\beta\in[0, 1]$ controls the trade-off between the two regularizers, and
$D_{\text{gw}}$ is the GW distance defined as follows.
\begin{definition}\label{def:gwd}
Let $(\mathcal{X}, d_x, p_x)$ and $(\mathcal{Y}, d_y, p_y)$ be two metric measure spaces, where $(\mathcal{X}, d_x)$ is a compact metric space and $p_x$ is a probability measure on $\mathcal{X}$ (with $(\mathcal{Y}, d_y, p_y)$ defined in the same way). 
The Gromov-Wasserstein distance $D_{\text{gw}}(p_x, p_y)$ is defined as
\begin{eqnarray*}
\begin{aligned}
&\sideset{}{_{\pi\in \Pi(p_x, p_y)}}\inf\sideset{}{_{\mathcal{X}\times \mathcal{Y}}}\smallint\sideset{}{_{\mathcal{X}\times \mathcal{Y}}}\smallint r_{x, y, x', y'}\mathrm{d}\pi(x, y)\mathrm{d}\pi(x',y')\\
=&\sideset{}{_{\pi\in \Pi(p_x, p_y)}}\inf\mathbb{E}_{(x,y,x',y')\sim \pi\times\pi}[r_{x, y, x', y'}]
\end{aligned}
\end{eqnarray*}
where $r_{x, y, x', y'}=|d_x(x, x')-d_y(y,y')|^2$, and $\Pi(p_x,p_y)$ is the set of all probability measures on $\mathcal{X}\times \mathcal{Y}$ with $p_x$ and $p_y$ as marginals.
\end{definition}
The $r_{x, y, x', y'}$ defines a relational loss, comparing the difference between the pairs of samples from the two distributions. 
Accordingly, the GW distance corresponds to the minimum expectation of the relational loss. 
The optimal joint distribution $\pi^*$ corresponding to the GW distance is called the optimal transport between the two distributions. 

The $D_{\text{gw}}(q_{z; Q}, p_z)$ in (\ref{eq:ae+prior+gw}) penalizes the structural difference between the two distributions, mutually enhancing the clustering structure of the prior and that of the posterior. 
We prefer using the GW distance to implement the relational regularizer, because of the ease by which it may be combined with existing regularizers, allowing design of scalable learning algorithms.
In particular, when the direct regularizer is the Wasserstein distance~\cite{titouan2019sliced}, $i.e.$, $D=D_{\text{w}}$, we can combine it with the $D_{\text{gw}}$ and derive a new regularizer as follows:
\begin{eqnarray}\label{eq:fgwd}
\begin{aligned}
&(1-\beta) D_{\text{w}}(p_x,p_y) + \beta D_{\text{gw}}(p_x, p_y)\\
=&(1-\beta)\sideset{}{_{\pi\in \Pi(p_x, p_y)}}\inf\sideset{}{_{\mathcal{X}\times \mathcal{Y}}}\smallint c_{x,y}\mathrm{d}\pi(x,y) + \\
 &\beta\inf_{\pi\in \Pi(p_x, p_y)}\sideset{}{_{\mathcal{X}\times \mathcal{Y}}}\smallint\sideset{}{_{\mathcal{X}\times \mathcal{Y}}}\smallint r_{x, y, x', y'}\mathrm{d}\pi(x, y)\mathrm{d}\pi(x',y')\\
\leq&\inf_{\pi\in \Pi(p_x, p_y)}\Bigl((1-\beta)\underbrace{\sideset{}{_{\mathcal{X}\times \mathcal{Y}}}\smallint c_{x,y}\mathrm{d}\pi(x,y)}_{\text{Wasserstein term}} + \\
 &\beta\underbrace{\sideset{}{_{\mathcal{X}\times \mathcal{Y}}}\smallint\sideset{}{_{\mathcal{X}\times \mathcal{Y}}}\smallint r_{x, y, x', y'}\mathrm{d}\pi(x, y)\mathrm{d}\pi(x',y')}_{\text{Gromov-Wasserstein term}}\Bigr)\\
=&D_{\text{fgw}}(p_x, p_y; \beta),
\end{aligned}
\end{eqnarray}
where $c:\mathcal{X}\times \mathcal{Y}\mapsto \mathbb{R}$ is a direct loss function between the two spaces. 
The new regularizer enforces a shared optimal transport for the Wasserstein and Gromov-Wasserstein terms, corresponding to the fused Gromov-Wasserstein (FGW) distance~\cite{vayer2018fused} between the distributions. 
The rationality of this combination has two perspectives.
First, the optimal transport indicates the correspondence between two spaces~\cite{memoli2011gromov,xu2019gromov}. 
In the following section, we show that this optimal transport maps encoded data to the clusters defined by the prior. 
Enforcing shared optimal transport helps ensure the consistency of the clustering structure. 
Additionally, as shown in (\ref{eq:fgwd}), $D_{\text{fgw}}(p_x, p_y; \beta)\geq (1-\beta) D_{\text{w}}(p_x,p_y) + \beta D_{\text{gw}}(p_x, p_y)$. 
When replacing the regularizers in (\ref{eq:ae+prior+gw}) with the FGW regularizer, we minimize an upper bound of the original objective function, useful from the viewpoint of optimization.

Therefore, we learn an autoencoder with relational regularization by solving the following optimization problem:
\begin{eqnarray}\label{eq:ae+fgw}
\min_{G,Q,p_z\in\mathcal{P}}\mathbb{E}_{p_x}\mathbb{E}_{q_{z|x;Q}}[d(x, G(z))] + \gamma D_{\text{fgw}}(q_{z;Q}, p_z;\beta),
\end{eqnarray}
where the prior $p_z$ is parameterized as a Gaussian mixture model (GMM) with $K$ components $\{\mathcal{N}(\mu_k,\Sigma_k)\}_{k=1}^{K}$. 
We set the probability of each component as $\frac{1}{K}$.
The autoencoder can be either probabilistic or deterministic, leading to different learning algorithms.

\section{Learning algorithms}
\subsection{Probabilistic autoencoder with hierarchical FGW}\label{sec:hfgw}
When the autoencoder is probabilistic, for each sample $x$, the encoder $Q$ outputs the mean and the logarithmic variance of the posterior $q_{z|x; Q}$. 
Accordingly, the marginal distribution $q_{z; Q}$ becomes a GMM as well, with number of components equal to the batch size, and the regularizer corresponds to the FGW distance between two GMMs. 
Inspired by the hierarchical Wasserstein distance~\cite{chen2018optimal,yurochkin2019hierarchical,lee2019hierarchical}, we leverage the structure of the GMMs and propose a hierarchical FGW distance to replace the original regularizer. 
In particular, given two GMMs, we define the hierarchical FGW distance between them as follows.
\begin{definition}[Hierarchical FGW]\label{def:hfgw}
Let $p=\sum_{k=1}^{K}a_kp_k$ and $q=\sum_{n=1}^{N}b_n q_n$ be two GMMs. 
$\{p_k\}_{k=1}^{K}$ and $\{q_n\}_{n=1}^{N}$ are $M$-dimensional Gaussian distributions.
$a=[a_k]\in\Delta^{K-1}$, $b=[b_n]\in\Delta^{N-1}$ are the distribution of the Gaussian components. 
For $\beta\in [0, 1]$, the hierarchical fused Gromov-Wasserstein distance between these two GMMs is
\begin{eqnarray}\label{eq:hfgw1}
\begin{aligned}
&D_{\text{hfgw}}(p,q; \beta)\\
=&\sideset{}{_{_{{T}=[t_{kn}]\in \Pi({a},{b})}}}\min(1-\beta)\sideset{}{_{k,n}}\sum D_{\text{w}}(p_k, q_n)t_{kn} + \\
&\beta\sideset{}{_{k,k',n,n'}}\sum |D_{\text{w}}(p_k, p_{k'}) - D_{\text{w}}(q_{n}, q_{n'})|^2 t_{kn}t_{k'n'}\\
=&D_{\text{fgw}}(a,b;\beta).
\end{aligned}
\end{eqnarray}
\end{definition}

As shown in (\ref{eq:hfgw1}), the hierarchical FGW corresponds to an FGW distance between the distributions of the Gaussian components, whose ground distance is the Wasserstein distance between the Gassuain components. 
Figures~\ref{fig:fgw} and~\ref{fig:hfgw} further illustrate the difference between the FGW and our hierarchical FGW.
For the two GMMs, instead of computing the optimal transport between them in the sample space, the hierarchical FGW builds optimal transport between their Gaussian components. 
Additionally, we have
\begin{proposition}
$D_{\text{hfgw}}(p, q; \beta) = D_{\text{fgw}}(p, q)$ when $\Sigma=0$ for all the Gaussian components.
\end{proposition}

\begin{figure}[t]
    \centering
    \subfigure[\tiny{FGW}]{
    \includegraphics[height=1.8cm]{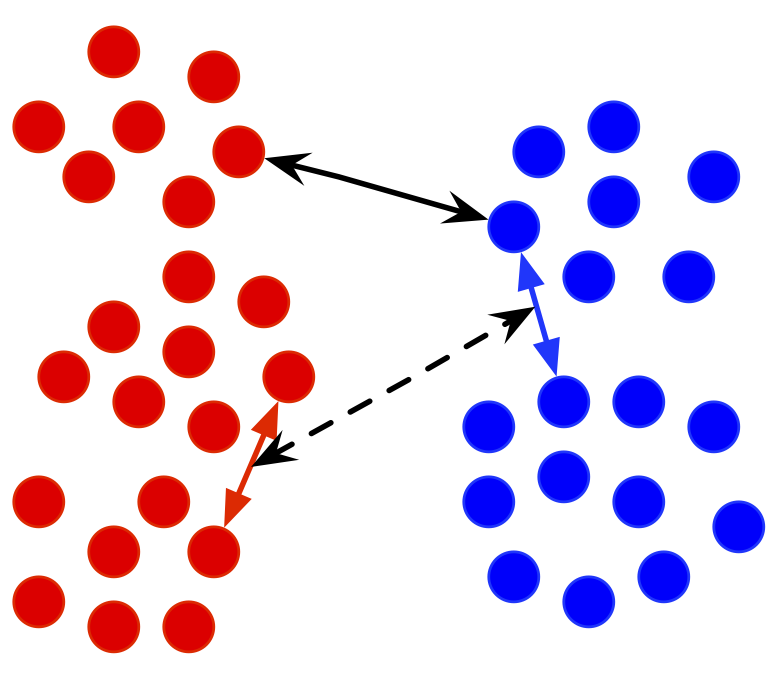}\label{fig:fgw}
    }\quad
    \subfigure[\tiny{Hierarchical FGW}]{
    \includegraphics[height=1.8cm]{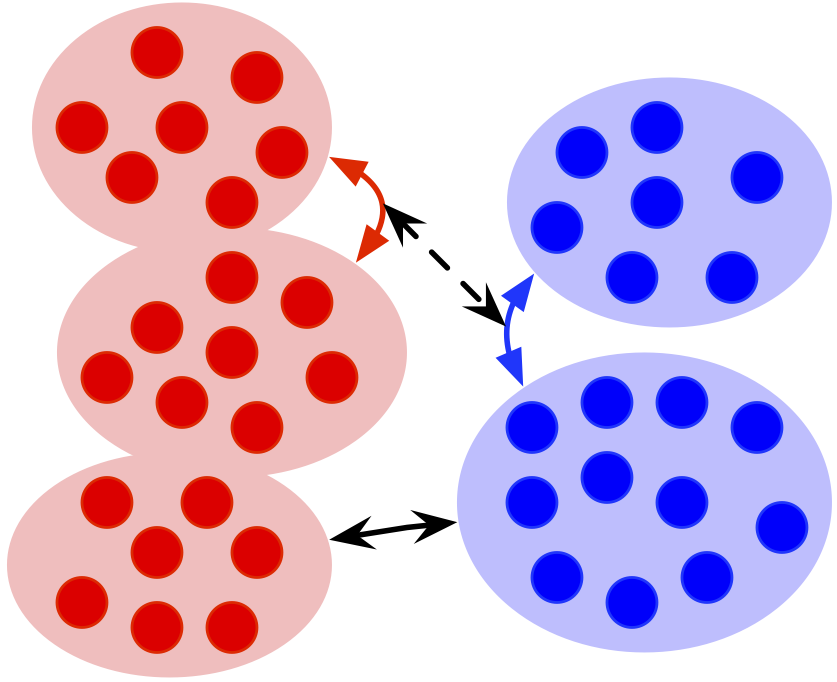}\label{fig:hfgw}
    }
    \subfigure[\tiny{Sliced FGW}]{
    \includegraphics[height=1.8cm]{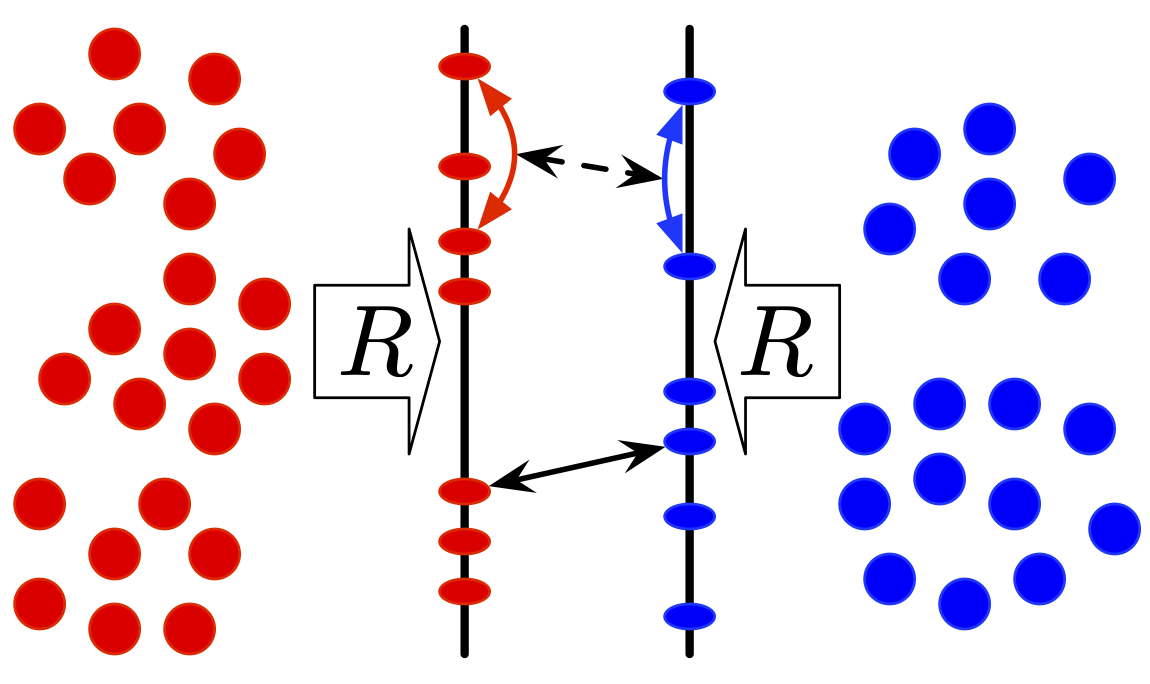}\label{fig:sfgw}
    }
    \vspace{-10pt}
    \caption{Illustrations of FGW, hierarchical FGW, and sliced FGW. The black solid arrow represents the distance between samples while the black dotted arrow represents the relational loss between sample pairs. In (a, c), the red and blue arrows represent the Euclidean distance between samples. In (b), the red and blue arrows represent the Wasserstein distance between Gaussian components.}
    \label{fig:cmp_fgw}
\end{figure}

Replacing the FGW with the hierarchical FGW, we convert an optimization problem of a continuous distribution (the $\pi$ in (\ref{eq:fgwd})) to a much simpler optimization problem of a discrete distribution (the $T$ in (\ref{eq:hfgw1})). 
Rewriting (\ref{eq:hfgw1}) in matrix form, we compute the hierarchical FGW distance via solving the following non-convex optimization problem:
\begin{eqnarray}\label{eq:hfgw2}
\begin{aligned}
D_{\text{hfgw}}(p,q; \beta) = \sideset{}{_{T\in \Pi(a,b)}}\min\langle D - 2\beta D_p T D_q^{\top}, T \rangle,
\end{aligned}
\end{eqnarray}
where $\langle\cdot,\cdot\rangle$ indicates the inner product between matrices,  $\Pi(a,b)=\{T\geq 0|T1_N=a,T^{\top}1_K=b\}$, and $1_N$ is an $N$-dimensional all-one vector. 
The optimal transport matrix $T$ is a joint distribution of the Gaussian components in the two GMMs, where  
$D_p =[D_{\text{w}}(p_k, p_{k'})]\in\mathbb{R}^{K\times K}$ and $D_q=[D_{\text{w}}(q_n, q_{n'})]\in\mathbb{R}^{N\times N}$, whose elements are the Wasserstein distances between Gaussian components, and 
\begin{eqnarray}\label{eq:D}
\begin{aligned}
D = (1-\beta)D_{pq} + \frac{\beta}{K}(D_p\odot D_p)
+\frac{\beta}{N}({D}_q\odot{D}_q)^{\top}
\end{aligned}
\end{eqnarray}
with $D_{pq} =[D_{\text{w}}(p_k, q_n)]\in\mathbb{R}^{K\times N}$ and $\odot$ represents the Hadamard product. 
The Wasserstein distance between Gaussian distributions has a closed form:
\begin{definition}\label{def:wg}
Let $p=\mathcal{N}(u_p, {\Sigma}_p)$ and $q=\mathcal{N}({u}_q, {\Sigma}_q)$ be two $N$-dimensional Gaussian distributions, where ${u}$ and ${\Sigma}$ represent the mean and the covariance matrix, respectively. 
The Wasserstein distance $D_{\text{w}}(p, q)$ is
\begin{eqnarray}\label{eq:wg}
\|{u}_p - {u}_q\|_2^2 + \text{trace}({\Sigma}_p+{\Sigma}_q - 2({\Sigma}_p^{\frac{1}{2}}{\Sigma}_q{\Sigma}_p^{\frac{1}{2}})^{\frac{1}{2}}).
\end{eqnarray}
\end{definition}
When the covariance matrices are diagonal, $i.e.$, ${\Sigma}=\text{diag}({\sigma}^2)$, where ${\sigma}=[\sigma_n]\in\mathbb{R}^N$ is the standard deviation, (\ref{eq:wg}) can be rewritten as
\begin{eqnarray}\label{eq:wg2}
D_{\text{w}}(p, q) = \|{u}_p - {u}_q\|_2^2 + \|{\sigma}_p - {\sigma}_q\|_2^2.
\end{eqnarray}
We solve (\ref{eq:hfgw2}) via the proximal gradient method in~\cite{xu2019gromov}, with further details in the Supplementary Material.

The hierarchical FGW is a good substitute for the original FGW, imposing structural constraints while being more efficient computationally.
Plugging the hierarchical FGW and its computation into (\ref{eq:ae+fgw}), we apply Algorithm~\ref{alg:ae+hfgw} to learn the proposed RAE. 
Note that taking advantage of the Envelope Theorem~\cite{afriat1971theory}, we treat the optimal transport matrix as constant when applying backpropagation, reducing computational complexity significantly.
The optimal transport matrix maps the components in the $q_{z; Q}$ to those in the $p_z$. 
Because the components in the $q_{z; Q}$ correspond to samples and the components in the $p_z$ correspond to clusters, this matrix indicates the clustering structure of the samples.

\begin{algorithm}[t]
\small{
	\caption{Learning RAE with hierarchical FGW}
	\label{alg:ae+hfgw}
	\begin{algorithmic}[1]
	    \STATE \textbf{Input} Samples in $\mathcal{X}$
	    \STATE \textbf{Output} The autoencoder $\{G, Q\}$ and the prior with $K$ Gaussian components $p_z=\frac{1}{K}\sum_k\mathcal{N}(z; \mu_k,\text{diag}(\sigma_k))$.
	    \STATE \textbf{for} each epoch
	    \STATE \quad\textbf{for} each batch of samples $\{x_n\}_{n=1}^{N}\subset \mathcal{X}$
	    \STATE \quad\quad $\mu_n, \log(\sigma_n^2)=Q(x_n)$ for $n=1,...,N$.
	    \STATE \quad\quad Reparameterize $z_n=\mu_n + \epsilon \sigma_n$, where $\epsilon\sim \mathcal{N}(0, I)$.
	    \STATE \quad\quad $q_{z;Q}=\frac{1}{N}\sum_n \mathcal{N}(z;\mu_n,\text{diag}(\sigma_n))$.
	    \STATE \quad\quad Calculate $D_p$, $D_q$ via (\ref{eq:wg2}), calculate $D$ via (\ref{eq:D})
	    \STATE \quad\quad Obtain optimal transport $T^*$ via solving (\ref{eq:hfgw2}).
	    \STATE \quad\quad $L_{\text{reconstruction}}=\sum_n d(x_n, G(z_n))$.
	    \STATE \quad\quad $D_{\text{hfgw}}(q_{z;Q}, p_z;\beta)=\langle D - 2\beta D_p T^* D_q^{\top}, T^* \rangle$.
	    \STATE \quad\quad Update $G, Q, p_z = \text{Adam}(L_{\text{reconstruction}} + \gamma D_{\text{hfgw}})$.
	\end{algorithmic}
}
\end{algorithm}

\subsection{Deterministic autoencoder with sliced FGW}
When the autoencoder is deterministic, its encoder outputs the latent codes corresponding to observed samples. 
These latent codes can be viewed as the samples of $q_{z; Q}$. 
For the prior $p_z$, we can also generate samples with the help of the reparameterization trick.
In such a situation, we estimate the FGW distance in (\ref{eq:ae+fgw}) based on the samples of the two distributions. 
For arbitrary two metric measure spaces $(\mathcal{X}, d_x, p_x)$ and $(\mathcal{Y}, d_y, p_y)$, the empirical FGW between their samples $\{x_i\}_{i=1}^N$ and $\{y_j\}_{j=1}^{N}$ is
\begin{eqnarray}\label{eq:efgw}
\begin{aligned}
&\widehat{D}_{\text{fgw}}(p_x, p_y;\beta) \\
=&\sideset{}{_{T\in\Pi(\frac{1}{N}1_N, \frac{1}{N}1_N)}}\min(1-\beta)\sideset{}{_{i,j=1}^{N}}\sum d(x_{i}, y_{j}) t_{ij} + \\
&\quad\quad\beta\sideset{}{_{i,i',j,j'}}\sum|d(x_{i}, x_{i'})-d(y_{j}, y_{j'})|^2t_{ij}t_{i'j'}.
\end{aligned}
\end{eqnarray}

We can rewrite this empirical FGW in matrix form as (\ref{eq:hfgw2}), and solve it by the proximal gradient method discussed above. 
When the samples are in 1D space and the metric is the Euclidean distance, however, according to the sliced GW distance in~\cite{titouan2019sliced} and the sliced Wasserstein distance in~\cite{kolouri2016sliced}, the optimal transport matrix corresponds to a permutation matrix and the $\widehat{D}_{\text{fgw}}(p_x, p_y;\beta)$ can be rewritten as:
\begin{eqnarray}\label{eq:efgw1D}
\begin{aligned}
&\widehat{D}_{\text{fgw}}(p_x, p_y;\beta)
=\min_{\sigma\in \mathcal{P}_N}\frac{1-\beta}{N}\sideset{}{_{i=1}^{N}}\sum(x_{i} - y_{\sigma(i)})^2+\\
&\quad\quad\quad\frac{\beta}{N}\sideset{}{_{i,j=1}^{N}}\sum((x_{i} - x_{j})^2 - (y_{\sigma(i)} - y_{\sigma(j)})^2)^2,
\end{aligned}
\end{eqnarray}
where $\mathcal{P}_N$ is the set of all permutations of $\{1,...,N\}$. 
Without loss of generality, we assume the 1D samples are sorted, $i.e.$, $x_1\leq ...\leq x_N$ and $y_1\leq ...\leq y_N$, and demonstrate that the solution of (\ref{eq:efgw1D}) is characterized by the following theorem.
\begin{theorem}\label{thm:solution}
For ${x},{y}\in\{{x}=[x_i],{y}=[y_j]\in\mathbb{R}^{N}\times\mathbb{R}^N|x_1\leq...\leq x_N, y_1\leq ...\leq y_N\}$, we denote their zero-mean translations as $x'$ and $y'$, respectively. 
The solution of (\ref{eq:efgw1D}) satisfies: 
1) When $(\sum_i x_i'y_i' +\frac{1-\beta}{8\beta})^2\geq (\sum_i x_i' y_{n+1-i}' +\frac{1-\beta}{8\beta})^2$, the solution is the identity permutation $\sigma(i)=i$. 
2) Otherwise, the solution is the anti-identity permutation $\sigma(i)=n+1-i$.
\end{theorem}
The proof of Theorem~\ref{thm:solution} is provided in the Supplementary Material. 
Consequently, for the samples in 1D space, we can calculate the empirical FGW distance via permuting the samples. 
To leverage this property for high-dimensional samples, we propose the following sliced FGW distance:
\begin{definition}[Sliced FGW]
Let $\mathcal{S}^{M-1}=\{\theta\in\mathbb{R}^M | \|\theta\|_2=1\}$ be the $M$-dimensional hypersphere and $u_{\mathcal{S}^{M-1}}$ the uniform measure on $\mathcal{S}^{M-1}$. 
For each $\theta$, we denote the projection on $\theta$ as $R_\theta$, where $R_{\theta}(x)=\langle x, \theta\rangle$. 
For $(\mathcal{X}, d_x, p_x)$ and $(\mathcal{Y}, d_y, p_y)$, we define their sliced fused Gromov-Wasserstein distance as
\begin{eqnarray*}
D_{\text{sfgw}}(p_x, p_y; \beta) = \mathbb{E}_{\theta\sim u_{\mathcal{S}^{M-1}}}[D_{\text{fgw}}(R_{\theta}\#p_x, R_{\theta}\#p_y;\beta])],
\end{eqnarray*}
where $R_{\theta}\#p$ represents the distribution after the projection, and $D_{\text{fgw}}(R_{\theta}\#p_x, R_{\theta}\#p_y)$ is the FGW distance between $(R_{\theta}(\mathcal{X}), d_{R_{\theta}(x)}, R_{\theta}\#p_x)$ and $(R_{\theta}(\mathcal{Y}), d_{R_{\theta}(y)}, R_{\theta}\#p_y)$.
\end{definition}

According to this definition, the sliced FGW projects the original metric measure spaces into 1D spaces, and calculates the FGW distance between these spaces. 
The sliced FGW corresponds to the expectation of the FGW distances under different projections.
We can approximate the sliced FGW distance based on the samples of the distributions as well.
In particular, given $\{x_i\}_{i=1}^{N}$ from $\mathcal{X}$, $\{y_i\}_{i=1}^{N}$ from $\mathcal{Y}$, and $L$ projections $\{R_{\theta_l}\}_{l=1}^{L}$, the empirical sliced FGW is
\begin{eqnarray}\label{eq:esfgw}
\begin{aligned}
&\widehat{D}_{\text{sfgw}}(p_x, p_y; \beta) \\
=&\frac{1}{L}\sideset{}{_{l=1}^{L}}\sum\widehat{D}_{\text{fgw}}(R_{\theta_l}\#p_x, R_{\theta_l}\#p_y;\beta)\\
=&\frac{1}{L}\sideset{}{_{l=1}^{L}}\sum\min_{\sigma\in \mathcal{P}_N}\frac{1-\beta}{N}\sideset{}{_{i=1}^{N}}\sum(x_{i,\theta_l} - y_{\sigma(i),\theta_l})^2+\\
&\frac{\beta}{N}\sideset{}{_{i,j=1}^{N}}\sum((x_{i,\theta_l} - x_{j,\theta_l})^2 - (y_{\sigma(i),\theta_l} - y_{\sigma(j),\theta_l})^2)^2,
\end{aligned}
\end{eqnarray}
where $x_{i,\theta_l}=R_{\theta_l}(x_i)$ represents the projected sample. 
Figure~\ref{fig:sfgw} further illustrates the principle of the sliced FGW distance. 
Replacing the empirical FGW with the empirical sliced FGW, we learn the relational regularized autoencoder via Algorithm~\ref{alg:ae+sfgw}.
\begin{algorithm}[t]
\small{
    \caption{Learning RAE with sliced FGW}	
    \label{alg:ae+sfgw}
	\begin{algorithmic}[1]
	    \STATE \textbf{Input} Samples in $\mathcal{X}$
	    \STATE \textbf{Output} The autoencoder $\{G, Q\}$ and the prior with $K$ Gaussian components $p_z=\frac{1}{K}\sum_k\mathcal{N}(z; \mu_k,\text{diag}(\sigma_k))$.
	    \STATE \textbf{for} each epoch
	    \STATE \quad\textbf{for} each batch of samples $\{x_n\}_{n=1}^{N}\subset \mathcal{X}$
	    \STATE \quad\quad\textbf{for} $n=1,...,N$
	    \STATE \quad\quad\quad Samples of $q_{z;Q}$: $z_n=Q(x_n)$.
	    \STATE \quad\quad\quad Samples of $p_z$: $k\sim \text{Categorical}(K)$, $z_{n}'=\mu_k + \epsilon\sigma_k$. 
	    \STATE \quad\quad \textbf{for} $l=1,...,L$
	    \STATE \quad\quad\quad Create a random projection $\theta_l\in\mathcal{S}^{M-1}$
	    \STATE \quad\quad\quad $z_{n,\theta_l}=R_{\theta_l}z_n$, $z_{n,\theta_l}'=R_{\theta_l}z_n'$ for $n=1,...,N$.
	    \STATE \quad\quad\quad Sort $\{z_{n,\theta_l}\}_{n=1}^{N}$ and $\{z_{n,\theta_l}'\}_{n=1}^{N}$, respectively.
	    \STATE \quad\quad\quad Calculate $\widehat{D}_{\text{fgw}}(R_{\theta_l}\#q_{z;Q}, R_{\theta_l}\#p_z;\beta])$ based on \\ \quad\quad\quad\quad sorted samples and Theorem~\ref{thm:solution}.
	    \STATE \quad\quad $L_{\text{reconstruction}}=\sum_n d(x_n, G(z_n))$.
	    \STATE \quad\quad Calculate $\widehat{D}_{\text{sfgw}}(q_{z;Q}, p_z;\beta)$ via (\ref{eq:esfgw}).
	    \STATE \quad\quad Update $G, Q, p_z = \text{Adam}(L_{\text{reconstruction}} + \gamma \widehat{D}_{\text{sfgw}})$.
	\end{algorithmic}
}
\end{algorithm}

\subsection{Comparisons on computational complexity}
Compared with calculating empirical FGW distance directly, our hierarchical FGW and sliced FGW have much lower computational complexity. 
Following notation in the previous two subsections, we denote the batch size as $N$, the number of Gaussian components in the prior as $K$, and the dimension of the latent code as $M$. 
If we apply the proximal gradient method in~\cite{xu2019gromov} to calculate the empirical FGW directly, the computational complexity is $\mathcal{O}(JN^3)$, where $J$ is the number of Sinkhorn iterations used in the algorithm. 
For our hierarchical FGW, we apply the proximal gradient method to a problem with a much smaller size ($i.e.$, solving (\ref{eq:hfgw2})) because of $K\ll N$ in general. 
Accordingly, the computational complexity becomes $\mathcal{O}(JN^2K)$.
For our sliced FGW, we apply $L$ random projections to project the latent codes to 1D spaces, whose complexity is $\mathcal{O}(LMN)$. 
For each pair of projected samples, we sort them with $\mathcal{O}(N\log N)$ operations and compute (\ref{eq:efgw1D}) with $O(N^2)$ operations.
Overall, the computational complexity of our sliced FGW is $\mathcal{O}(LN(M+\log N + N))$. 
Because $J\approx L$ in general, the computational complexity of the sliced FGW is comparable to that of the hierarchical FGW.

\section{Relational Co-Training of Autoencoders}
Besides learning a single autoencoder, we can apply our relational regularization to learn multiple autoencoders. 
As shown in Figure~\ref{fig:scheme_2}, when learning two autoencoders we can penalize the GW distance between their posterior distributions, and accordingly the learning problem becomes:
\begin{eqnarray}\label{eq:cotrain}
\begin{aligned}
&\min_{\{G_s,Q_s\}_{s=1}^{2}}\sideset{}{_{s=1}^{2}}\sum\Bigl(\mathbb{E}_{p_{x_s}}\mathbb{E}_{q_{z_s|x_s; Q_s}}[d(x_s, G_s(z_s))]+\\
&\gamma (1-\tau)D(q_{z_s;Q_s},p_{z_s})\Bigr) +2\gamma\tau D_{\text{gw}}(q_{z_1;Q_1}, q_{z_2;Q_2}).
\end{aligned}
\end{eqnarray}
The regularizer $D$ quantifies the discrepancy between the marginalized posterior and the prior, the prior distributions can be predefined or learnable parameters, $\tau\in[0, 1]$ achieves a trade-off between $D$ and the relational regularizer $D_{\text{gw}}$, and  
$\gamma$ controls the overall significance of these two kinds of regularizers. 
When learning probabilistic autoencoders, we set $D$ to the hierarchical Wasserstein distance between GMMs~\cite{chen2018optimal} and approximate the relational regularizer by a hierarchical GW distance, equivalent to the hierarchical FGW with $\beta=1$.
When learning deterministic autoencoders, we set $D$ to the sliced Wasserstein distance used in~\cite{kolouri2018sliced} and approximate the relational regularizer via the sliced GW~\cite{titouan2019sliced} (the sliced FGW with $\beta=1$).

The main advantage of the proposed relational regularization is that it is applicable for co-training heterogeneous autoencoders. 
As shown in (\ref{eq:cotrain}), the data used to train the autoencoders can come from different domains and with different data distributions. 
To fully capture the information in each domain, sometimes the autoencoders have heterogeneous architectures, and the corresponding latent codes are in incomparable spaces, $e.g.$, with different dimensions. 
Taking the GW distance as the relational regularizer, we impose a constraint on the posterior distributions defined in different latent spaces, encouraging structural similarity between them. 
This regularizer helps avoid over-regularization because it does not enforce a shared latent distribution across different domains.
Moreover, the proposed regularizer is imposed on the posterior distributions. 
In other words, it does not require samples from different domains to be paired. 

According to the analysis above, our relational co-training strategy has potential for multi-view learning, especially in the scenario with unpaired samples. 
In particular, given the data in different domains, we first learn their latent codes via solving (\ref{eq:cotrain}). 
Concatenating the latent codes in different domains, we can use the concatenation of the latent codes as the features for downstream learning tasks.

\section{Related Work}
\textbf{Gromov-Wasserstein distance}
The GW distance has been used as a metric for shape registration~\cite{memoli2009spectral,memoli2011gromov}, vocabulary set alignment~\cite{alvarez2018gromov}, and graph matching~\cite{chowdhury2018gromov,vayer2018optimal,xu2019gromov}.
The work in~\cite{peyre2016gromov} proposes an entropy-regularized GW distance and calculates it based on Sinkhorn iterations~\cite{cuturi2013sinkhorn}. 
Following this direction, the work in~\cite{xu2019gromov} replaces the entropy regularizer with a Bregman proximal term. 
The work in~\cite{xu2019gwf} proposes an ADMM-based method to calculate the GW distance. 
To further reduce the computational complexity, the recursive GW distance~\cite{xu2019scalable} and the sliced GW distance~\cite{titouan2019sliced} have been proposed. 
For generative models, the work in~\cite{bunne2019learning} leverages the GW distance to learn coupled adversarial generative networks. 
However, none of the existing autoencoders consider using the GW distance as their regularizer.

\textbf{Autoencoders}
The principle of the autoencoder is to minimize the discrepancy between the data and model distributions. 
The common choices of the discrepancy include the KL divergence~\cite{kingma2013auto,dilokthanakul2016deep,tomczak2018vae,takahashi2019variational} and the Wasserstein distance~\cite{tolstikhin2018wasserstein,kolouri2018sliced}, which lead to different learning algorithms. 
Our relational regularized autoencoder can be viewed as a new member of the Wasserstein autoencoder family.
Compared with the MMD and the GAN loss used in WAE~\cite{tolstikhin2018wasserstein}, and the sliced Wasserstein distance used in~\cite{kolouri2018sliced}, our FGW-based regularizer imposes relational constraints and allows the learning of an autoencoder with structured prior distribution. 

\begin{table}[t]
\begin{small}
    \caption{Comparisons for different autoencoders}\label{tab:mix}
    \centering
    \begin{tabular}{
    @{\hspace{3pt}}c@{\hspace{3pt}}|
    @{\hspace{3pt}}c@{\hspace{3pt}}
    @{\hspace{3pt}}c@{\hspace{3pt}}
    @{\hspace{3pt}}c@{\hspace{3pt}}}
    \hline\hline
    Method    
    &$Q:\mathcal{X}\mapsto\mathcal{Z}$ 
    &$p_z$ 
    &$D(q_{z;Q},p_z)$\\ \hline
    VAE
    &Probabilistic
    &$\mathcal{N}(z;0, I)$
    &KL\\
    WAE
    &Deterministic
    &$\mathcal{N}(z;0, I)$
    &MMD\\
    SWAE
    &Deterministic
    &$\mathcal{N}(z;0, I)$
    &$D_{\text{w}}$\\
    GMVAE     
    &Probabilistic          
    &$\frac{1}{K}\sum_{k}\mathcal{N}(z;{u}_k,{\Sigma}_k)$ 
    &KL \\
    VampPrior 
    &Probabilistic          
    &$\frac{1}{K}\sum_{k}\mathcal{N}(z;Q(x_k))$                 
    &KL \\ \hline
    \multirow{2}{*}{\textbf{Our RAE}} 
    &Probabilistic  
    &\multirow{2}{*}{$\frac{1}{K}\sum_{k}\mathcal{N}(z;{u}_k,{\Sigma}_k)$}
    &$D_{\text{hfgw}}$\\
    &Deterministic 
    &
    &$\widehat{D}_{\text{sfgw}}$\\
    \hline\hline
    \end{tabular}
\end{small}
\end{table}

\textbf{Co-training methods}
For data in different domains, a commonly-used co-training strategy maps them to a shared latent space, and encourages  similarity between their latent codes.
This co-training strategy suppresses the risk of overfitting for each model and enhances their generalization power, which achieves encouraging performance in multi-view learning~\cite{kumar2011co,chen2017multi,sindhwani2005co}. 
However, this strategy assumes that the latent codes yield the same distribution, which may lead to over regularization.
Additionally, it often requires well-aligned data, $i.e.$, the samples in different domains are paired. 
Our relational co-training strategy provides a potential solution to relax these restrictions for practical applications. 

\section{Experiments}
\subsection{Image generation}
We test our relational regularized autoencoder (\textbf{RAE}) for image-generation tasks and compare it with the following alternatives: the variational autoencoder (VAE)~\cite{kingma2013auto}, the Wasserstein autoencoder (WAE)~\cite{tolstikhin2018wasserstein}, the sliced Wasserstein autoencoder (SWAE)~\cite{kolouri2018sliced}, the Gaussian mixture VAE (GMVAE)~\cite{dilokthanakul2016deep}, and the VampPrior~\cite{tomczak2018vae}. 
Table~\ref{tab:mix} lists the main differences between our RAE and these baselines. 

We test the methods on the MNIST~\cite{lecun1998gradient} and CelebA datasets~\cite{liu2015faceattributes}. 
For fairness, all the autoencoders have the same DCGAN-style architecture~\cite{radford2015unsupervised} and are learned with the same hyperparameters: the learning rate is $0.001$; the optimizer is Adam~\cite{kingma2014adam} with $\beta_1=0.5$ and $\beta_2=0.999$; the number of epochs is $50$; the batch size is $100$; the weight of regularizer $\gamma$ is $1$; the dimension of latent code is $8$ for MNIST and $64$ for CelebA. 
For the autoencoders with structured priors, we set the number of the Gaussian components to be $10$ and initialize their prior distributions at random. 
For the proposed RAE, the hyperparameter $\beta$ is set to be $0.1$, which empirically makes the Wasserstein term and the GW term in our FGW distance have the same magnitude. 
The probabilistic RAE calculates the hierarchical FGW based on the proximal gradient method with $20$ iterations, and the deterministic RAE calculates the sliced FGW with $50$ random projections.
All the autoencoders use Euclidean distance as the distance between samples, thus the reconstruction loss is the mean-square-error (MSE). 
We implement all the autoencoders with PyTorch and train them on a single NVIDIA GTX 1080 Ti GPU.
More implementation details, $e.g.$, the architecture of the autoencoders, are provided in Supplementary Material.

\begin{table}[t]
\centering
\begin{small}
    \caption{Comparisons on learning image generator}\label{tab:gen_image}
    \begin{tabular}{
    @{\hspace{3pt}}c@{\hspace{3pt}}|
    @{\hspace{3pt}}c@{\hspace{3pt}}|
    c@{\hspace{3pt}}c@{\hspace{3pt}}|
    c@{\hspace{3pt}}c@{\hspace{3pt}}}
    \hline\hline
    Encoder  & \multirow{2}{*}{Method} &\multicolumn{2}{c|}{MNIST}  &\multicolumn{2}{c}{CelebA} \\
    $Q:\mathcal{X}\mapsto\mathcal{Z}$&  &Rec. loss &FID &Rec. loss &FID\\
    \hline
    \multirow{4}{*}{Probabilistic} &VAE &16.60     &156.11 &96.36  &59.99\\
    &GMVAE                              &16.76     &60.88  &108.13 &353.17\\
    &VampPrior                          &22.41     &127.81 &---    &---\\
    &RAE                                &14.14     &41.99  &63.21  &52.20\\
    \hline
    \multirow{3}{*}{Deterministic} &WAE &9.97      &52.78  &63.83  &52.07\\
    &SWAE                               &11.10     &35.63  &87.02  &88.91\\
    &RAE                                &10.37     &49.39  &64.49  &51.45\\
    \hline\hline
    \end{tabular}
\end{small}
\end{table}

\begin{table}[t]
\centering
\begin{small}
    \caption{Runtime per epoch (second) when training various models}\label{tab:time}
    \begin{tabular}{
    @{\hspace{3pt}}c@{\hspace{5pt}}|
    @{\hspace{5pt}}c@{\hspace{5pt}}
    @{\hspace{5pt}}c@{\hspace{5pt}}
    @{\hspace{5pt}}c@{\hspace{5pt}}
    @{\hspace{5pt}}c@{\hspace{3pt}}}
    \hline\hline
    Dataset &WAE   &SWAE  &P-RAE &D-RAE\\ \hline
    MNIST   &25.6  &24.7  &26.9  &24.8\\
    CelebA  &602.2 &553.4 &618.5 &569.7\\
    \hline\hline
    \end{tabular}
\end{small}
\end{table}

\begin{figure*}[t]
    \centering
    \subfigure[MNIST]{
    \includegraphics[height=4.8cm]{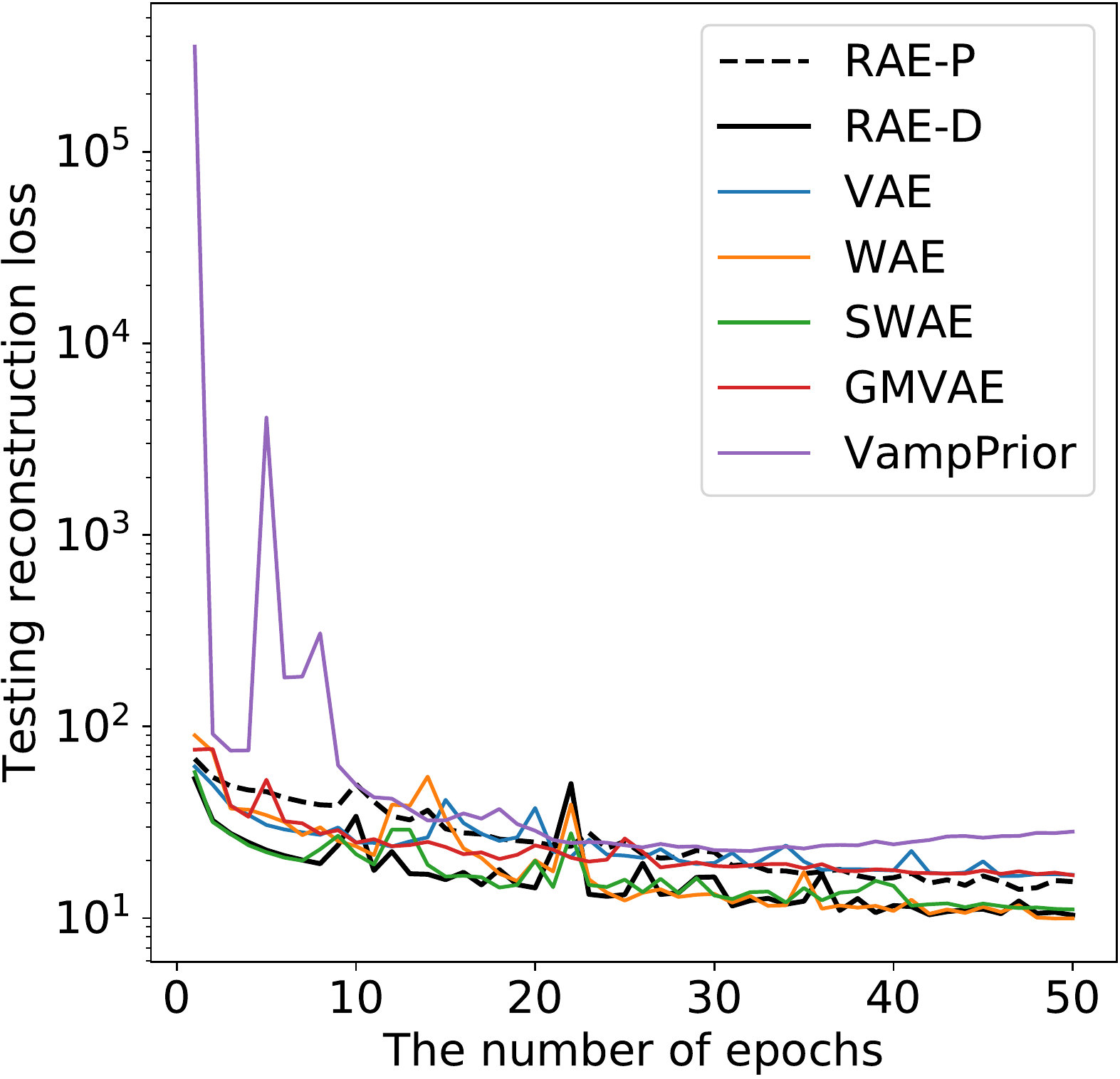}\label{fig:converge1}
    }
    \subfigure[Enlarged (a)]{
    \includegraphics[height=4.8cm]{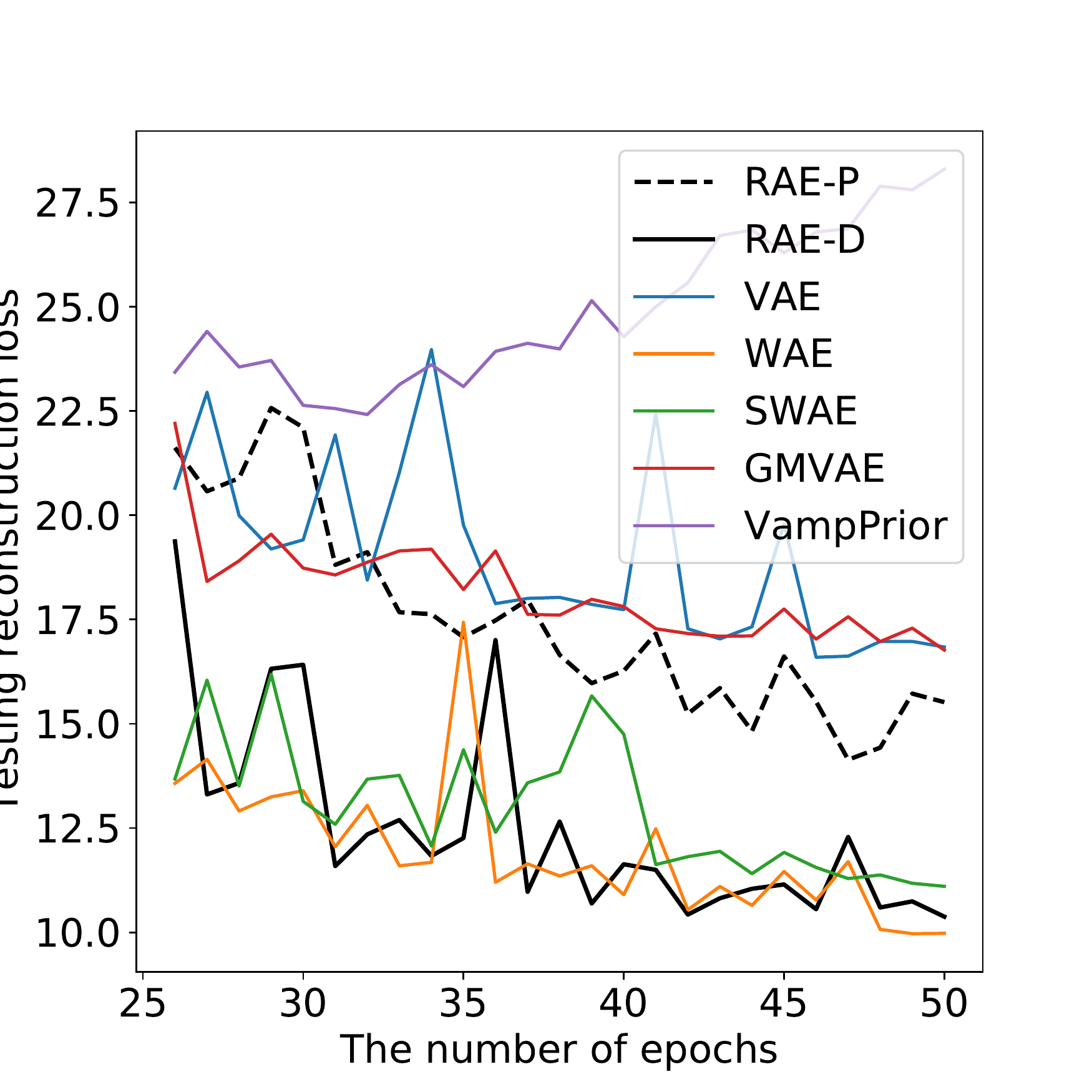}\label{fig:converge2}
    }
    \subfigure[CelebA]{
    \includegraphics[height=4.8cm]{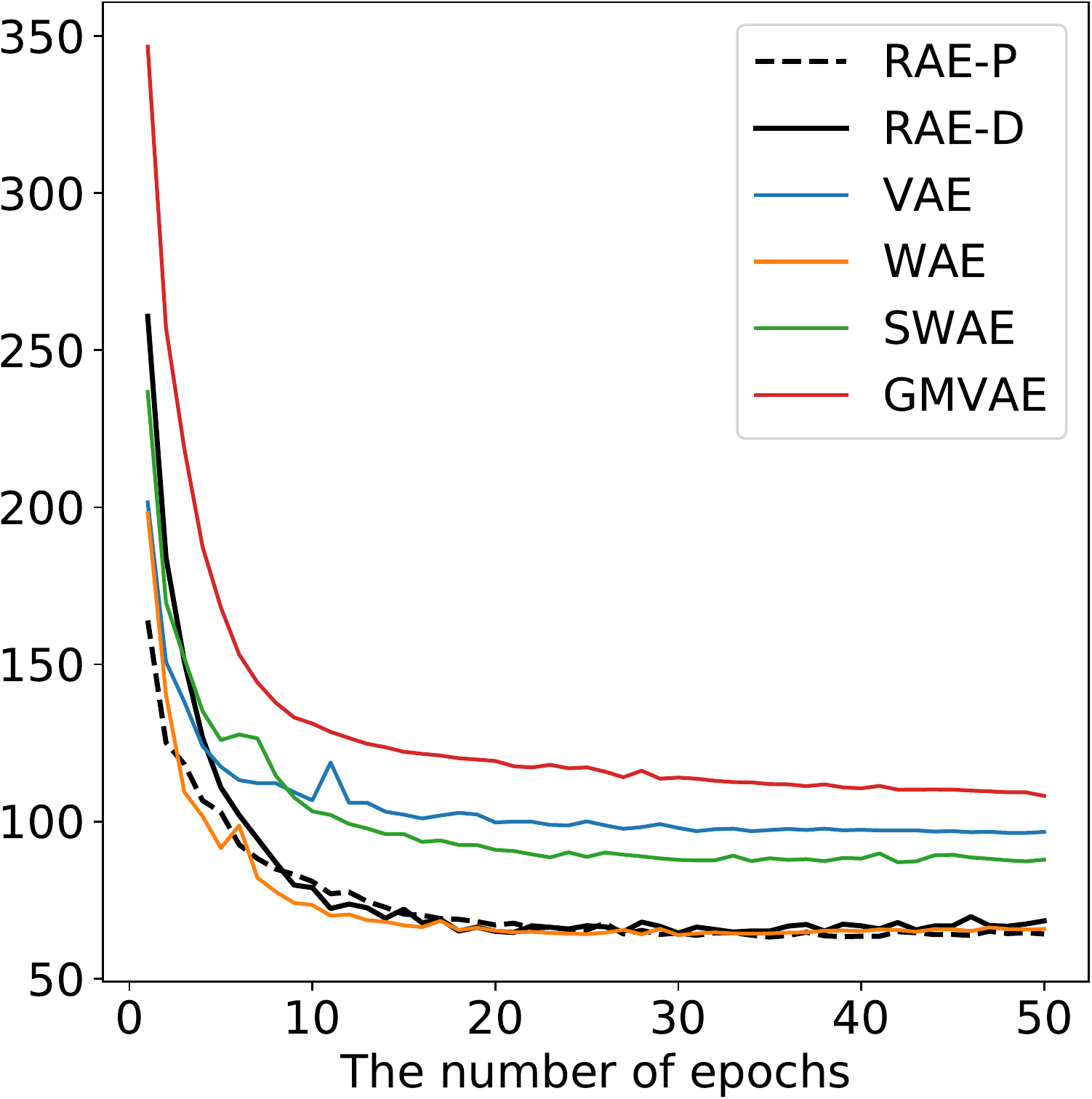}
    }
    \vspace{-10pt}
    \caption{Comparisons for various methods on their convergence.}
    \label{fig:converge}
\end{figure*}

\begin{figure*}[t]
    \centering
    \includegraphics[width=1\linewidth]{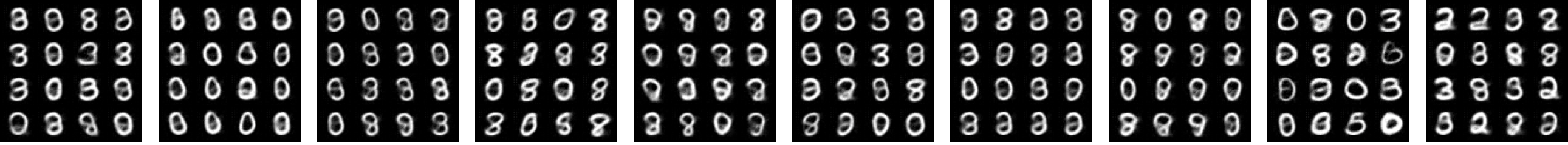}\\
    \vspace{-3pt}\text{(a) VampPrior}
    \includegraphics[width=1\linewidth]{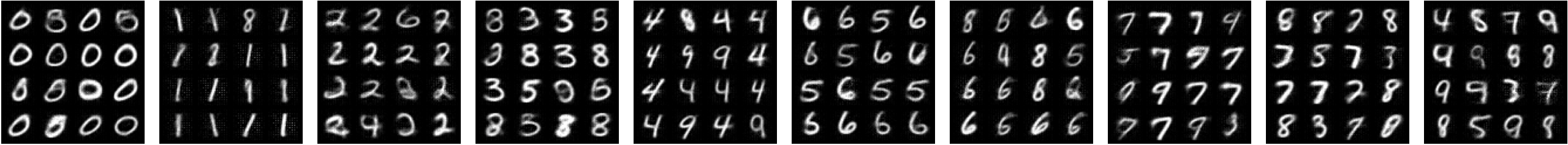}\\
    \vspace{-3pt}\text{(b) GMVAE}
    \includegraphics[width=1\linewidth]{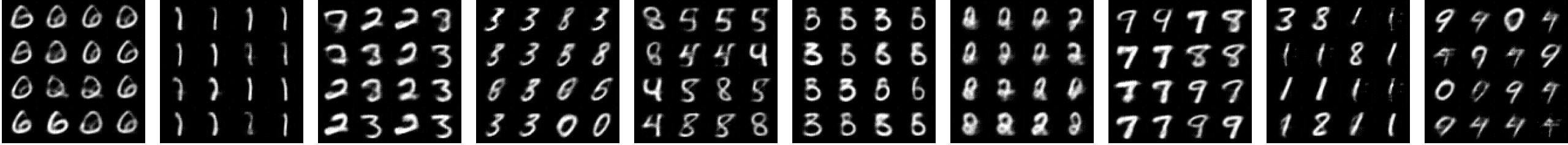}\\
    \vspace{-3pt}\text{(c) Probabilistic RAE}
    \includegraphics[width=1\linewidth]{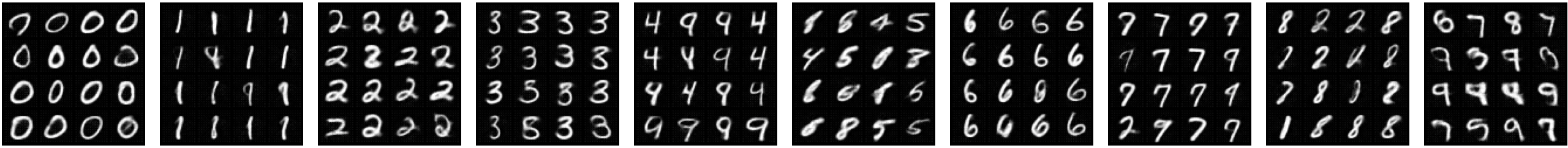}\\
    \vspace{-3pt}\text{(d) Deterministic RAE}
    \vspace{-10pt}
    \caption{Comparisons on conditional digit generation.}
    \label{fig:mnist_cg}
\end{figure*}

For each dataset, we compare the proposed RAE with the baselines on $i$) the reconstruction loss on testing samples; $ii$) the Fr\'echet Inception Distance (FID) between 10,000 testing samples and 10,000 randomly generated samples. 
We list the performance of various autoencoders in Table~\ref{tab:gen_image}.
Among probabilistic autoencoders, our RAE consistently achieves the best performance on both testing reconstruction loss and FID score. 
When learning deterministic autoencoders, our RAE is at least comparable to the considered alternatives on these measurements. 
Figure~\ref{fig:converge} compares the autoencoders on their convergence of the reconstruction loss. 
The convergence of our RAE is almost the same as that of state-of-the-art methods, which further verifies its feasibility.

\begin{figure*}[t]
    \centering
    \includegraphics[width=0.95\linewidth]{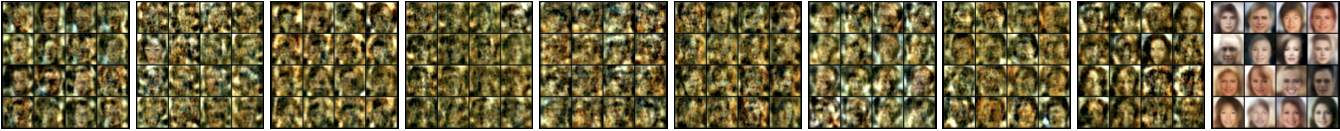}\\
    \vspace{-3pt}\text{(a) GMVAE}
    \includegraphics[width=0.95\linewidth]{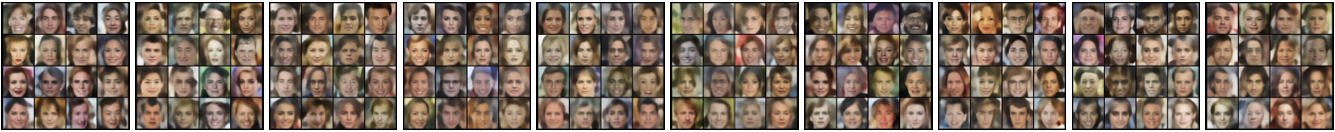}\\
    \vspace{-3pt}\text{(b) Probabilistic RAE}
    \includegraphics[width=0.95\linewidth]{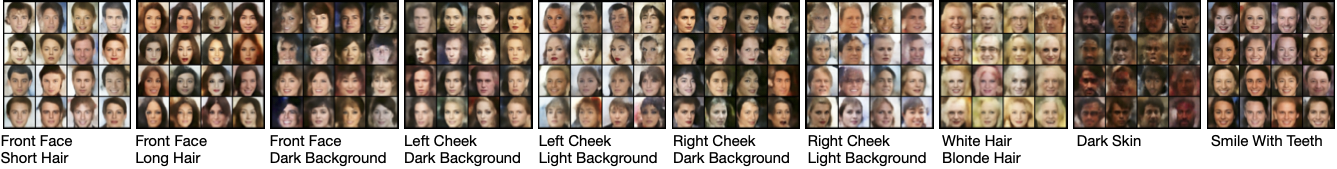}\\
    \vspace{-3pt}\text{(c) Deterministic RAE}
    \vspace{-10pt}
    \caption{Comparisons on conditional face generation.}
    \label{fig:celeba_cg}
\end{figure*}

For the autoencoders learning GMMs as their priors, we further make comparisons for them in conditional generation tasks, $i.e.$, generating samples conditioned on specific Gaussian components. 
Figures~\ref{fig:mnist_cg} and~\ref{fig:celeba_cg} visualize the generation results for various methods. 
For the MNIST dataset, the GMVAE, our probabilistic RAE, and deterministic RAE achieve desired generation results. 
The images conditioned on different Gaussian components correspond to different digits/writing styles. 
The VampPrior, however, suffers from a problem of severe mode collapse. 
The images conditioned on different Gaussian components are similar to each other and with limited modes -- most of them are ``0'', ``2'',  ``3'', and ``8''. 
As shown in Table~\ref{tab:mix}, for each Gaussian component of the prior, the VampPrior parameterizes it by passing a landmark $x_k$ through the encoder. 
Because the landmarks are in the sample space, this implicit model requires more parameters, making it sensitive to initialization with a high risk of overfitting. 
Figures~\ref{fig:converge1} and \ref{fig:converge2} verify our claim: the testing loss of the VampPrior is unstable and does not converge well during training. 
For the CelebA dataset, the GMVAE fails to learn a GMM-based prior. 
As shown in Figure~\ref{fig:celeba_cg}(a), the GMVAE trains a single Gaussian distribution, while ignoring the remaining components. 
As a result, only one Gaussian component can generate face images. 
Our probabilistic and deterministic RAE, by contrast, learn their GMM-based prior successfully. 
In particular, all the components of our probabilistic RAE can generate face images, but the components are indistinguishable. 
Our deterministic RAE achieves the best performance in this conditional generation task -- different components can generate semantically-meaningful images with interpretable diversity. 
For each component, we add some tags to highlight semantic meaning.
The visual comparisons for various autoencoders on their reconstructed and generated samples are shown in Supplementary Material.

\subsection{Multi-view learning via co-training autoencoders}
We test our relational co-training strategy on four multi-view learning datasets~\cite{li2015large}:\footnote{\url{https://github.com/yeqinglee/mvdata}} \textbf{Caltech101-7} is a subset of the Caltech-101 dataset~\cite{fei2004learning} with 1,474 images in $7$ classes. 
Each image is represented by $48$-dimensional Gabor features and $40$-dimensional Wavelet moments.
\textbf{Caltech101-20} is a subset of the Caltech-101 with 2,386 images in $20$ classes. The features are the same with the Caltech101-7. 
\textbf{Handwritten} is a dataset with 2,000 images corresponding to $10$ digits. 
Each image has $240$-dimensional pixel-based features and $76$-dimensional Fourier coefficients.
\textbf{Cathgen} is a real-world dataset of 8,000 patients. For each patient, we seek to leverage $44$-dimensional clinical features and $67$-dimensional genetic features to predict the happening of myocardial infarction, $i.e.$, a binary classification task.

\begin{table*}[t]
\centering
\begin{small}
    \caption{Comparisons on classification accuracy (\%)}\label{tab:mvl}
    \begin{tabular}{
    @{\hspace{4pt}}l@{\hspace{4pt}}|
    @{\hspace{4pt}}c@{\hspace{4pt}}|
    @{\hspace{4pt}}c@{\hspace{4pt}}|
    @{\hspace{4pt}}c@{\hspace{4pt}}|
    @{\hspace{4pt}}c@{\hspace{4pt}}|
    @{\hspace{4pt}}c@{\hspace{4pt}}}
    \hline\hline
    Method &Data type &Caltech101-7 &Caltech101-20 &Handwritten  &Cathgen\\
    \hline
    Independent AEs &Unpaired 
    &56.92$\pm$1.67 
    &33.07$\pm$2.09  
    &52.09$\pm$5.82  
    &64.36$\pm$1.93 \\
    AEs+CoReg~\cite{sindhwani2005co}      &Paired   
    &76.58$\pm$1.38  
    &60.25$\pm$1.66  
    &56.20$\pm$5.25  
    &66.79$\pm$1.30 \\
    CCA~\cite{via2007learning}            &Paired   
    &78.33$\pm$1.88  
    &52.27$\pm$2.32  
    &66.28$\pm$5.02  
    &65.28$\pm$2.17 \\
    AEs+CCA~\cite{wang2015deep}        &Paired   
    &80.24$\pm$1.22  
    &62.37$\pm$1.35  
    &69.72$\pm$4.64  
    &66.89$\pm$1.57 \\    
    AEs+W          &Unpaired 
    &83.07$\pm$1.69  
    &\textbf{69.58}$\pm$2.03  
    &71.21$\pm$5.55  
    &66.06$\pm$1.68 \\
    AEs+GW (Ours)       &Unpaired 
    &\textbf{84.29}$\pm$1.74  
    &69.39$\pm$2.01  
    &\textbf{72.36}$\pm$3.82  
    &\textbf{66.99}$\pm$1.77 \\
    \hline\hline
    \end{tabular}
\end{small}
\end{table*}

For each dataset, we use 80\% of the data for training, 10\% for validation, and the remaining 10\% for testing. 
We test various multi-view learning methods. 
For each method, we first learn two autoencoders for the data in different views in an unsupervised way, and then concatenate the latent codes of the autoencoders as the features and train a classifier based on softmax regression.
When learning autoencoders, our relational co-training method solves (\ref{eq:cotrain}) with $\gamma=1$ and $\tau=0.5$. 
The influence of $\tau$ on the learning results is shown in Supplementary Material.
For simplification, we set the prior distributions as normal distributions in (\ref{eq:cotrain}). 
The autoencoders are probabilistic, whose encoders and decoders are MLPs. 
Each autoencoder has $20$-dimensional latent codes, and 
more implementation details are provided in Supplementary Material.
We set $D$ as the hierarchical Wasserstein distance and the relational regularizer as the hierarchical GW distance. 
In addition to the proposed method, denoted as \textbf{AEs+GW}, we consider the following baselines: $i$) learning two variational autoencoders independently (\textbf{Independent AEs}); $ii$) learning two variational autoencoders jointly with a least-square co-regularization~\cite{sindhwani2005co} (\textbf{AEs+CoReg}); $iii$) learning latent representations via canonical correlation analysis (\textbf{CCA})~\cite{via2007learning}; $iv$) learning two autoencoders jointly with a CCA-based regualrization (\textbf{AEs+CCA})~\cite{wang2015deep}; $v$) learning two autoencoders by replacing the $D_{\text{gw}}$ in (\ref{eq:cotrain}) with a Wasserstein regularizer (\textbf{AEs+W}). 
The AE+CoReg penalizes the Euclidean distance between the latent codes from different views, which needs paired samples. 
The remaining methods penalize the discrepancy between the distributions of the latent codes, which are applicable for unpaired samples.
The classification accuracy in Table~\ref{tab:mvl} demonstrates effectiveness of our relational co-training strategy, as the proposed method outperforms the baselines consistently across different datasets.

\section{Conclusions}
A new framework has been proposed for learning autoencoders with relational regularization. 
Leveraging the GW distance, this framework allows the learning of structured prior distributions associated with the autoencoders and prevents the model from under-regularization.
Besides learning a single autoencoder, the proposed relational regularizer is beneficial for co-training heterogeneous autoencoders.
In the future, we plan to make this relational regularizer applicable for co-training more than two autoencoders and further reduce its computational complexity. 

\textbf{Acknowledgements} This research was supported in part by DARPA, DOE, NIH, ONR and NSF. We
thank Dr. Hongyuan Zha for helpful discussions. 

\newpage
\bibliography{rae}
\bibliographystyle{icml2020}

% %%%%%%%%%%%%%%%%%%%%%%%%%%%%%%%%%%%%%%%%%%%%%%%%%%%%%%%%%%%%%%%%%%%%%%%%%%%%%%%
% %%%%%%%%%%%%%%%%%%%%%%%%%%%%%%%%%%%%%%%%%%%%%%%%%%%%%%%%%%%%%%%%%%%%%%%%%%%%%%%
% % DELETE THIS PART. DO NOT PLACE CONTENT AFTER THE REFERENCES!
% %%%%%%%%%%%%%%%%%%%%%%%%%%%%%%%%%%%%%%%%%%%%%%%%%%%%%%%%%%%%%%%%%%%%%%%%%%%%%%%
% %%%%%%%%%%%%%%%%%%%%%%%%%%%%%%%%%%%%%%%%%%%%%%%%%%%%%%%%%%%%%%%%%%%%%%%%%%%%%%%
\newpage
\onecolumn
\appendix
\section{The proximal gradient method}\label{app:1}
Both the hierarchical FGW in~(\ref{eq:hfgw1}) and the empirical FGW in~(\ref{eq:efgw}) can be rewritten in matrix format.
As shown in (\ref{eq:hfgw2}), the calculation of the distance corresponds to solving the following non-convex optimization problem:
\begin{eqnarray}
\sideset{}{_{T\in \Pi(p,q)}}\min\langle D - 2 D_p T D_q^{\top}, T \rangle,
\end{eqnarray}
where $p$ and $q$ are predefined discrete distributions.
This problem can be solved iteratively by the following proximal gradient method~\cite{xu2019gromov}.
In each $j$-th iteration, given current estimation $T^{(j)}$, we consider the following problem with a proximal term
\begin{eqnarray}
\sideset{}{_{T\in \Pi(p,q)}}\min\langle D - 2 D_p T^{(j)} D_q^{\top}, T \rangle + \alpha \text{KL}(T\| T^{(j)}).
\end{eqnarray}
This subproblem can be solved easily via Sinkhorn iterations~\cite{cuturi2013sinkhorn}.
The details of the algorithm are shown in Algorithm~\ref{alg:proximal}.
In our experiments, we set $J=20$ when learning probabilistic RAE (P-RAE). 
The hyperparameter $\alpha$ is set adaptively. 
In particular, in each iteration, given the matrix $C^{(j)}=D-2D_p T^{(j)}D_q^{\top}$, we set $\alpha=0.1\max(C^{(j)})$.
This setting helps us improve the numerical stability when calculating the $\Phi$ in Algorithm~\ref{alg:proximal}.
When learning deterministic RAE (D-RAE), we apply the sliced FGW distance with $L=50$ random projections, such that the training time of the D-RAE and that of the P-RAE are comparable.

\begin{algorithm}[h]
\small{
    \caption{$\min_{T\in \Pi(p,q)}\langle D - 2 D_p T D_q^{\top}, T \rangle$}	
    \label{alg:proximal}
	\begin{algorithmic}[1]
	    \STATE Initialize ${T}^{(0)}=pq^{\top}$, $a=p$
		\STATE \textbf{for} $j=0,...,J-1$ 
		\STATE \quad $\Phi=\exp(-\frac{1}{\alpha}(D-2D_p{T}^{(j)}D_q^{\top}))\odot T^{(j)}$.
		\STATE \quad Sinkhorn iteration: $\bm{b}=\frac{q}{{\Phi}^{\top}{a}}$,  ${a} = \frac{{p}}{{\Phi}{b}}$,
		\STATE \quad $T^{(j+1)} = \text{diag}({a}){\Phi}\text{diag}({b})$.
		\STATE \textbf{Return} ${T}^{(J)}$
	\end{algorithmic}
}
\end{algorithm}

\section{The Proof of Theorem~\ref{thm:solution}}\label{app:2}

\begin{theorem}\label{thm:trans1}
For ${x},{y}\in\mathcal{I}$, where $\mathcal{I}=\{{x}=[x_i],{y}=[y_i]\in\mathbb{R}^{N}\times\mathbb{R}^N|x_1\leq...\leq x_N, y_1\leq ...\leq y_N\}$, the solution of the problem
\begin{eqnarray}\label{eq:prob}
\sideset{}{_{\sigma\in \mathcal{P}_N}}\min \sideset{}{_{i,j}}\sum((x_i - x_j)^2 - (y_{\sigma(i)} - y_{\sigma(j)})^2)^2 + \alpha \sideset{}{_{i}}\sum(x_i - y_{\sigma(i)})^2,
\end{eqnarray}
where $\mathcal{P}_N$ is the set of all permutation of $\{1,...,N\}$, is invariant with respect to any translations of ${x}$ and ${y}$.
\end{theorem}
\begin{proof}
Denote the translations of ${x}$ and ${y}$ as ${x}'={x}+t_x{1}_N$ and ${y}'={y}+t_y{1}_N$, respectively, where $t_x,t_y\in\mathbb{R}$.
We then denote the objective function in (\ref{eq:prob}) as $F({x},{y},\sigma)$. 
Accordingly, we have
\begin{eqnarray}\label{eq:equiv}
\begin{aligned}
F({x}',{y}',\sigma)
= &\sideset{}{_{i,j}}\sum((x_i+t_x - x_j-t_x)^2 - (y_{\sigma(i)}+t_y - y_{\sigma(j)}-t_y)^2)^2\\
&+ \alpha \sideset{}{_{i}}\sum(x_i +t_x- y_{\sigma(i)}-t_y)^2\\
=&\sideset{}{_{i,j}}\sum((x_i - x_j)^2 - (y_{\sigma(i)} - y_{\sigma(j)})^2)^2 \\
&+\alpha \sideset{}{_{i}}\sum\Bigl[(x_i- y_{\sigma(i)})^2 + (t_x - t_y)^2 + 2(x_i-y_{\sigma(i)})(t_x-t_y)\Bigr]\\
=&F({x},{y},\sigma) + \alpha N(t_x-t_y)^2 + 2\alpha(t_x-t_y)\sideset{}{_{i}}\sum(x_i - y_{\sigma(i)})\\
=&F({x},{y},\sigma) + \underbrace{\alpha N(t_x-t_y)^2 + 2\alpha(t_x-t_y)(X-Y)}_{\text{not dependent on $\sigma$}},
\end{aligned}
\end{eqnarray}
where $X=\sum_i x_i$, $Y=\sum_i y_{\sigma(i)}=\sum_i y_i$. 
Based on (\ref{eq:equiv}), we have
\begin{eqnarray}
\sideset{}{_{\sigma\in \mathcal{P}_N}}\min 
F({x}', {y}',\sigma)=\text{Constant} + \sideset{}{_{\sigma\in \mathcal{P}_N}}\min F({x}, {y},\sigma),
\end{eqnarray}
whose solution is invariant.
\end{proof}

\begin{theorem}\label{thm:solution1}
Following the notations in Theorem~\ref{thm:trans1}, for ${x},{y}\in\mathcal{I}$ we denote their zero-mean translations as ${x}'={x}+t_x{1}_N$ and ${y}'={y}+t_y{1}_N$, respectively. 
$\arg\min_{\sigma\in\mathcal{P}_N} F({x},{y},\sigma)$ satisfies:\\
1) When $(\sum_i x_i'y_i' +\frac{\alpha}{8})^2\geq (\sum_i x_i' y_{n+1-i}' +\frac{\alpha}{8})^2$, the solution is the identity permutation $\sigma(i)=i$.\\
2) Otherwise, the solution is the anti-identity permutation $\sigma(i)=n+1-i$.
\end{theorem}
\begin{proof}
The proposed problem is equivalent to the following problem
\begin{eqnarray}
\sideset{}{_{\sigma\in \mathcal{P}_N}}\max Z({x},{y},\sigma) =\sideset{}{_{\sigma\in \mathcal{P}_N}}\max\sideset{}{_{i,j}}\sum(x_i-x_j)^2(y_{\sigma(i)}-y_{\sigma(j)})^2 + \alpha\sideset{}{_{i}}\sum x_i y_{\sigma(i)},
\end{eqnarray}
and we denote $X=\sum_i x_i$ and $Y=\sum_i y_i=\sum_i y_{\sigma(i)}$.
Accordingly, we have
\begin{eqnarray}
\begin{aligned}
\sideset{}{_{\sigma\in\mathcal{P}_N}}\max Z({x},{y},\sigma)
=&\sideset{}{_{\sigma\in\mathcal{P}_N}}\max\sideset{}{_{i,j}}\sum(x_i-x_j)^2(y_{\sigma(i)}-y_{\sigma(j)})^2 + \alpha\sideset{}{_{i}}\sum x_i y_{\sigma(i)}\\
=&\sideset{}{_{\sigma\in\mathcal{P}_N}}\max \sideset{}{_{i,j}}\sum (x_i^2 + x_j^2)(y_{\sigma(i)}^2 + y_{\sigma(j)}^2) -2\sideset{}{_{i,j}}\sum x_ix_j(y_{\sigma(i)}^2 + y_{\sigma(j)}^2)\\
&\quad\quad\quad\quad - 2\sideset{}{_{i,j}}\sum(x_i^2 + x_j^2)y_{\sigma(i)}y_{\sigma(j)} + 4\sideset{}{_{i,j}}\sum x_ix_jy_{\sigma(i)}y_{\sigma(j)} +\alpha\sideset{}{_i}\sum x_iy_{\sigma(i)}\\
=&\sideset{}{_{\sigma\in\mathcal{P}_N}}\max 2N\sideset{}{_i}\sum x_i^2y_{\sigma(i)}^2 +2\sideset{}{_i}\sum x_i^2\sideset{}{_i}\sum y_i^2
-2\sideset{}{_{i,j}}\sum x_ix_j(y_{\sigma(i)}^2 + y_{\sigma(j)}^2) \\
&\quad\quad\quad\quad - 2\sideset{}{_{i,j}}\sum(x_i^2 + x_j^2)y_{\sigma(i)}y_{\sigma(j)} + 4\sideset{}{_{i,j}}\sum x_ix_jy_{\sigma(i)}y_{\sigma(j)} +\alpha\sideset{}{_i}\sum x_iy_{\sigma(i)}\\
=&\sideset{}{_{\sigma\in\mathcal{P}_N}}\max 2N\sideset{}{_i}\sum x_i^2y_{\sigma(i)}^2 +2\sideset{}{_i}\sum x_i^2\sideset{}{_i}\sum y_i^2
-4X\sideset{}{_i}\sum x_i y_{\sigma(i)}^2 \\
&\quad\quad\quad\quad - 4Y\sideset{}{_i}\sum x_i^2 y_{\sigma(i)} + 4\sideset{}{_{i,j}}\sum x_ix_jy_{\sigma(i)}y_{\sigma(j)} +\alpha\sideset{}{_i}\sum x_iy_{\sigma(i)}\\ 
=&2\Bigl( \sideset{}{_{\sigma\in\mathcal{P}_N}}\max \sideset{}{_i}\sum N x_i^2y_{\sigma(i)}^2  -2X\sideset{}{_i}\sum x_i y_{\sigma(i)}^2 \\
&\quad\quad\quad\quad - 2Y\sideset{}{_i}\sum x_i^2 y_{\sigma(i)} + 2(\sideset{}{_i}\sum x_i y_{\sigma(i)})^2 +\frac{\alpha}{2}\sideset{}{_i}\sum x_iy_{\sigma(i)}\Bigr) + 2\sideset{}{_i}\sum x_i^2\sideset{}{_i}\sum y_i^2, 
\end{aligned}
\end{eqnarray}
where the last term $2\sideset{}{_i}\sum x_i^2\sideset{}{_i}\sum y_i^2$ does not depend on $\sigma$.
Therefore, we define
\begin{eqnarray}\label{eq:W}
W({x},{y},\sigma):=\sideset{}{_i}\sum f(x_i, y_{\sigma(i)}) +\frac{\alpha}{2}\sideset{}{_i}\sum x_i y_{\sigma(i)}+ 2(\sideset{}{_i}\sum x_i y_{\sigma(i)})^2,
\end{eqnarray}
where 
\begin{eqnarray}
f(x_i,y_{\sigma(i)}):=Nx_i^2y_{\sigma(i)}^2-2X x_i y_{\sigma(i)}^2 - 2Y x_i^2 y_{\sigma(i)},
\end{eqnarray}
such that $\forall {x},{y}\in\mathcal{I}$
\begin{eqnarray}
\begin{aligned}
\arg\sideset{}{_{\sigma\in\mathcal{P}_N}}\max Z({x},{y},\sigma)
=&\arg\sideset{}{_{\sigma\in\mathcal{P}_N}}\max W({x},{y},\sigma)\\
=&\arg\sideset{}{_{\sigma\in\mathcal{P}_N}}\max N\sideset{}{_i}\sum f(x_i,y_{\sigma(i)}) + \frac{\alpha}{2}\sideset{}{_i}\sum x_i y_{\sigma(i)}+ 2(\sideset{}{_i}\sum x_i y_{\sigma(i)})^2.
\end{aligned}
\end{eqnarray}
Furthermore, we define a translated version of $\sum_i f(x_i, y_{\sigma(i)})$ as
\begin{eqnarray}
g({x}, {y}, b) = \sideset{}{_i}\sum f(x_i + b, y_{\sigma(i)}).
\end{eqnarray}
According to the translation invariance shown in Theorem~\ref{thm:trans1}, we assume $X=Y=0$ and want to find a translation $b^*$ such that $g({x},{y},b^*)=0$ and simply the problem. 
Specifically, we have
\begin{eqnarray}
\begin{aligned}
g({x},{y},b)
=& \sideset{}{_i}\sum f(x_i + b, y_{\sigma(i)})\\
=&\sideset{}{_i}\sum N(x_i+b)^2y_{\sigma(i)}^2-2(X+Nb)(x_i+b) y_{\sigma(i)}^2 - 2Y(x_i+b)^2 y_{\sigma(i)}\\
=&\sideset{}{_i}\sum N(x_i^2 + 2bx_i +b^2)y_{\sigma(i)}^2-2(Nbx_i + Nb^2) y_{\sigma(i)}^2\\
=&N\Bigl(\sideset{}{_i}\sum x_i^2y_{\sigma(i)}^2 -b^2\sideset{}{_i}\sum y_{\sigma(i)}^2\Bigr).
\end{aligned}
\end{eqnarray}
Obviously, $b^*=\pm\frac{\sum_i x_i^2y_{\sigma(i)}^2}{\sum_i y_{\sigma(i)}^2}$ makes $g({x},{y},b^*)=0$.
Plugging $X=Y=0$ and $b^*$ into (\ref{eq:W}), we have
\begin{eqnarray}\label{eq:finalobj}
\begin{aligned}
W({x}+b^*{1}_N,{y},\sigma)&=g({x},{y},b^*)+ +\frac{\alpha}{2}\sideset{}{_i}\sum (x_i+b^*) y_{\sigma(i)}+ 2(\sideset{}{_i}\sum (x_i+b^*) y_{\sigma(i)})^2\\
&=\frac{\alpha}{2}\sideset{}{_i}\sum x_i y_{\sigma(i)}+ 2(\sideset{}{_i}\sum x_i y_{\sigma(i)})^2\\
&=2\Bigl(\sideset{}{_i}\sum x_i y_{\sigma(i)} + \frac{\alpha}{8}\Bigr)^2 - \frac{\alpha^2}{32}.
\end{aligned}
\end{eqnarray}
In summary, the original problem is equivalent to $\max_{\sigma\in\mathcal{P}_N}W({x}+b^*{1}_N,{y},\sigma)$. 

For any ${x}, {y}\in \mathcal{I}$, we recall the rearrangement inequality:
\begin{eqnarray}\label{eq:rearrange}
\forall \sigma\in\mathcal{P}_N,\quad \sideset{}{_i}\sum x_i y_{n+1-i}\leq \sideset{}{_i}\sum x_i y_{\sigma(i)}\leq \sideset{}{_i}\sum x_i y_i
\end{eqnarray}
Based on (\ref{eq:rearrange}), the solution of $\max_{\sigma\in\mathcal{P}_N}W({x}+b^*{1}_N,{y},\sigma)$ satisfies:
\begin{itemize}
    \item When $(\sum_i x_iy_i +\frac{\alpha}{8})^2\geq (\sum_i x_i y_{n+1-i} +\frac{\alpha}{8})^2$, the solution is the identity permutation $\sigma(i)=i$.
    \item Otherwise, the solution is the anti-identity permutation $\sigma(i)=n+1-i$.
\end{itemize}
\end{proof}
In our case, $\alpha = \frac{1-\beta}{\beta}$.

\section{Implementation Details}\label{app:3}
For the MNIST dataset, the architecture of the autoencoders is
\begin{eqnarray*}
&\text{Encoder}~Q:\quad &x\in\mathbb{R}^{28\times 28}\rightarrow\text{Conv}_{\text{128}}\rightarrow\text{Conv}_{\text{256}}\rightarrow\text{Conv}_{\text{512}}\rightarrow\text{Conv}_{\text{1024}}\rightarrow\text{FC}_{\text{8}}\rightarrow z\in\mathbb{R}^8\\
&\text{Decoder}~G:\quad &z\in\mathbb{R}^8\rightarrow\text{FC}_{7\times 7\times 1024}\rightarrow\text{FSConv}_{\text{512}}\rightarrow\text{FSConv}_{\text{256}}\rightarrow\text{FSConv}_{\text{1}}\rightarrow x\in\mathbb{R}^{28\times 28},
\end{eqnarray*}
where $\text{Conv}_k$ stands for a convolution with $k$ $4\times 4$ filters, $\text{FSConv}_k$ for the fractional strided convolution with $k$ $4\times 4$ filters, and $\text{FC}_k$ for the fully connected layer mapping to $\mathbb{R}^k$.
Except the last layer of the decoder, each of the remaining convolution layers is followed by the batch normalization and the ReLU activation. 
For probabilistic autoencoders, the last layer of the encoder contains two $\text{FC}$ layers, outputting the mean and the logarithmic variance accordingly.

For the CelebA dataset, the architecture of the autoencoders is
\begin{eqnarray*}
&\text{Encoder}~Q:\quad &x\in\mathbb{R}^{64\times 64\times 3}\rightarrow\text{Conv}_{\text{128}}\rightarrow\text{Conv}_{\text{256}}\rightarrow\text{Conv}_{\text{512}}\rightarrow\text{Conv}_{\text{1024}}\rightarrow\text{FC}_{\text{64}}\rightarrow z\in\mathbb{R}^{64}\\
&\text{Decoder}~G:\quad &z\in\mathbb{R}^{64}\rightarrow\text{FC}_{8\times 8\times 1024}\rightarrow\text{FSConv}_{\text{512}}\rightarrow\text{FSConv}_{\text{256}}\rightarrow\text{FSConv}_{\text{128}}\rightarrow\text{FSConv}_{\text{3}}\rightarrow x\in\mathbb{R}^{64\times 64\times 3},
\end{eqnarray*}
where $\text{Conv}_k$ stands for a convolution with $k$ $5\times 5$ filters, $\text{FSConv}_k$ for the fractional strided convolution with $k$ $5\times 5$ filters, and $\text{FC}_k$ for the fully connected layer mapping to $\mathbb{R}^k$.
Except the last layer of the decoder, each of the remaining convolution layers is followed by the batch normalization and the ReLU activation. 
For probabilistic autoencoders, the last layer of the encoder contains two $\text{FC}$ layers, outputting the mean and the logarithmic variance accordingly. 

For the multi-view learning experiments, the architecture of the autoencoders is
\begin{eqnarray*}
&\text{Encoder}~Q:\quad &x\in\mathbb{R}^{V}\rightarrow\text{FC}_{\text{50}}+\text{ReLU}\rightarrow\text{FC}_{\text{20}}\rightarrow z\in\mathbb{R}^{20}\\
&\text{Decoder}~G:\quad &z\in\mathbb{R}^{20}\rightarrow\text{FC}_{\text{50}}+\text{ReLU}\rightarrow\text{FC}_{V}\rightarrow x\in\mathbb{R}^{V},
\end{eqnarray*}
where $V$ is the dimension of the data.
For probabilistic autoencoders, the last layer of the encoder contains two $\text{FC}$ layers, outputting the mean and the logarithmic variance accordingly. 

It should be noted that our relational co-training strategy is feasible when the latent codes of different views are with different dimensions. 
Here, we set the dimension as $20$ for all the views to guarantee the fairness on comparisons. 
In such a situation, both our relational regularizer and the traditional regularizers used in AE+CoReg and AE+W are applicable. 
Accordingly, the comparisons among them demonstrate the superiority of our strategy.

\begin{figure}[t]
    \centering
    \subfigure[Caltech101-7]{
    \includegraphics[width=0.22\textwidth]{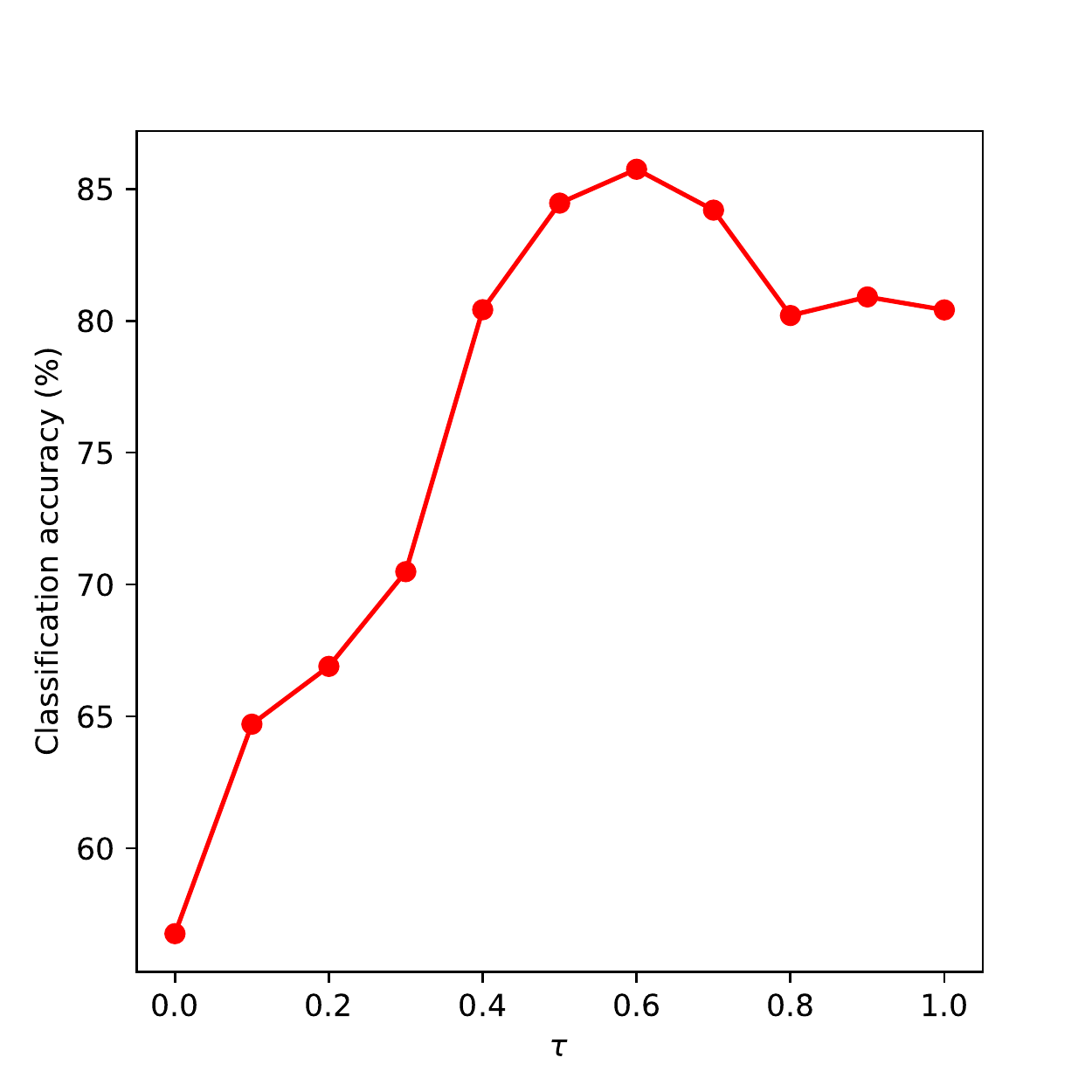}
    }
    \subfigure[Caltech101-20]{
    \includegraphics[width=0.22\textwidth]{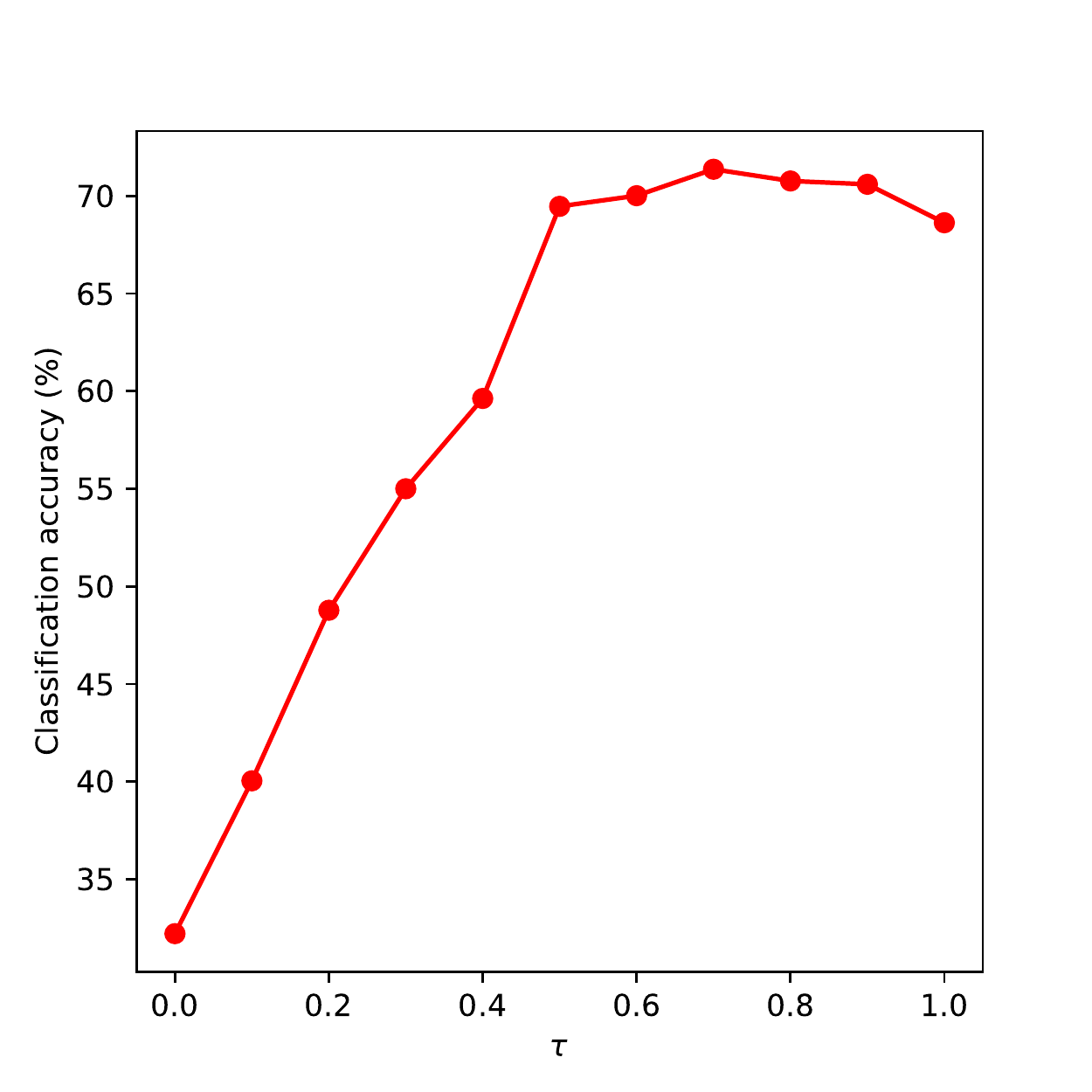}
    }
    \subfigure[Handwritten]{
    \includegraphics[width=0.22\textwidth]{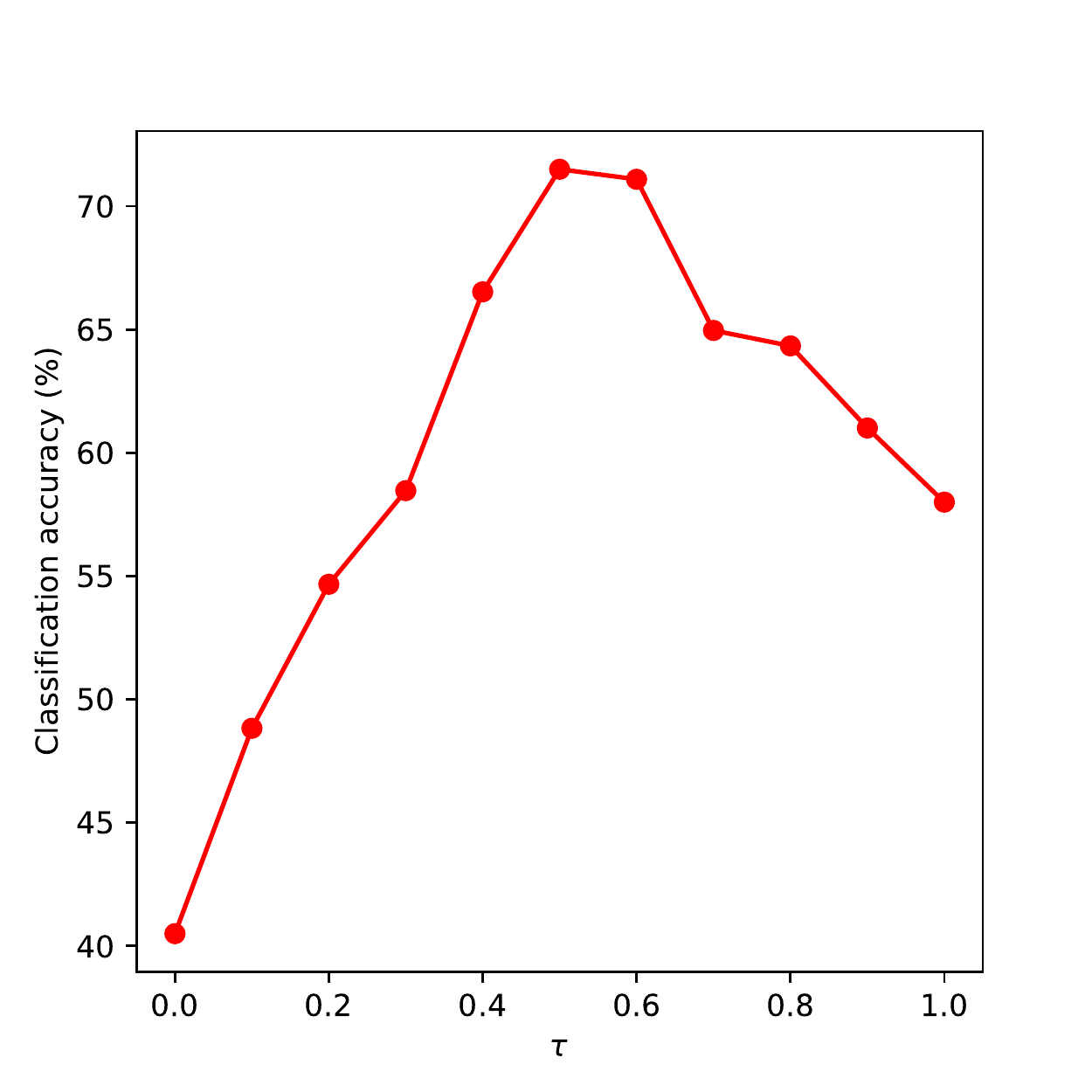}
    }
    \subfigure[Cathgen]{
    \includegraphics[width=0.22\textwidth]{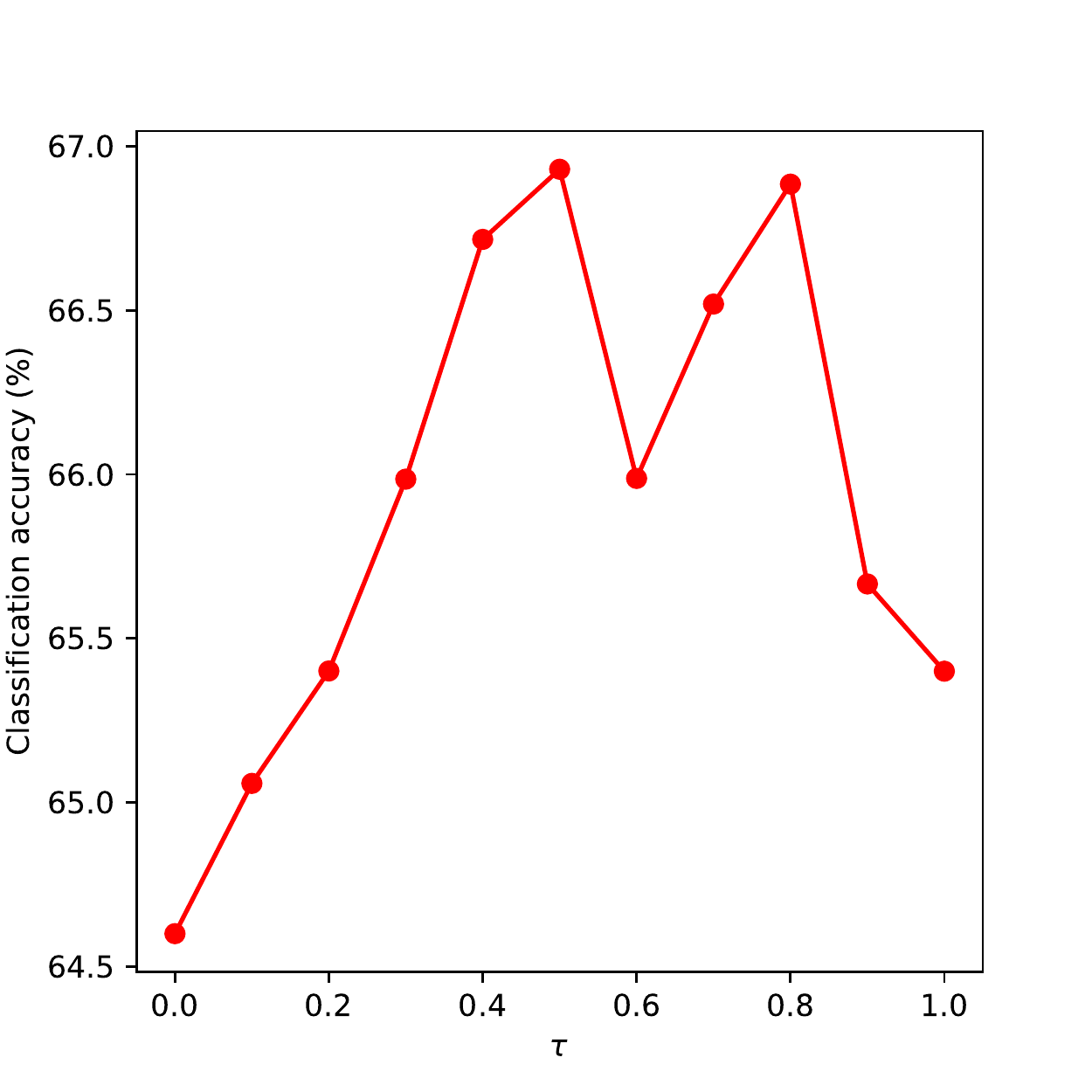}
    }
    \vspace{-10pt}
    \caption{The influences of $\tau$ on the classification accuracy for various datasets.}
    \label{fig:cmp_sample2}
\end{figure}

\begin{figure}[t]
    \centering
    \subfigure[MNIST]{
    \includegraphics[height=2.5cm]{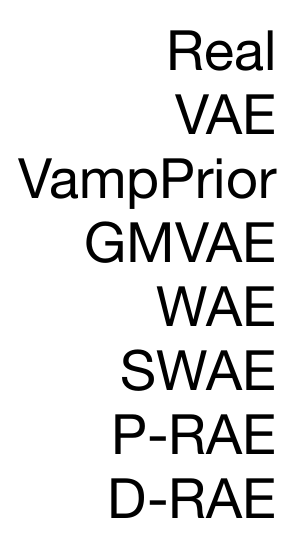}
    \includegraphics[height=2.5cm]{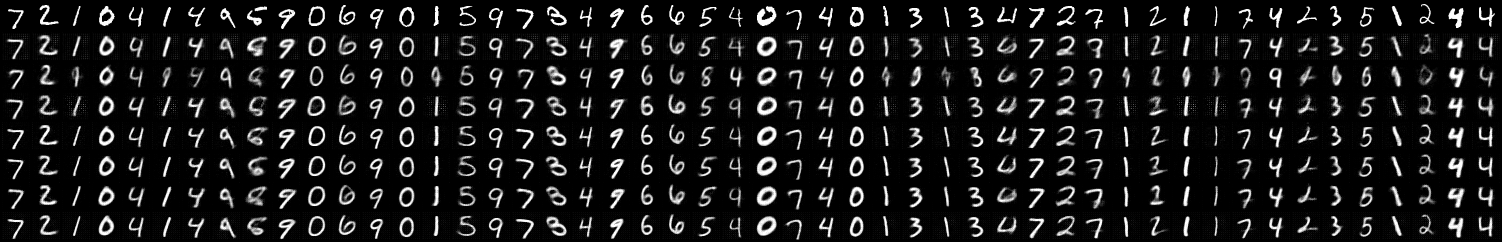}
    }
    \subfigure[CelebA]{
    \includegraphics[height=5cm]{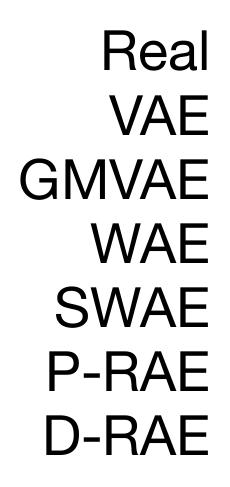}
    \includegraphics[height=5cm]{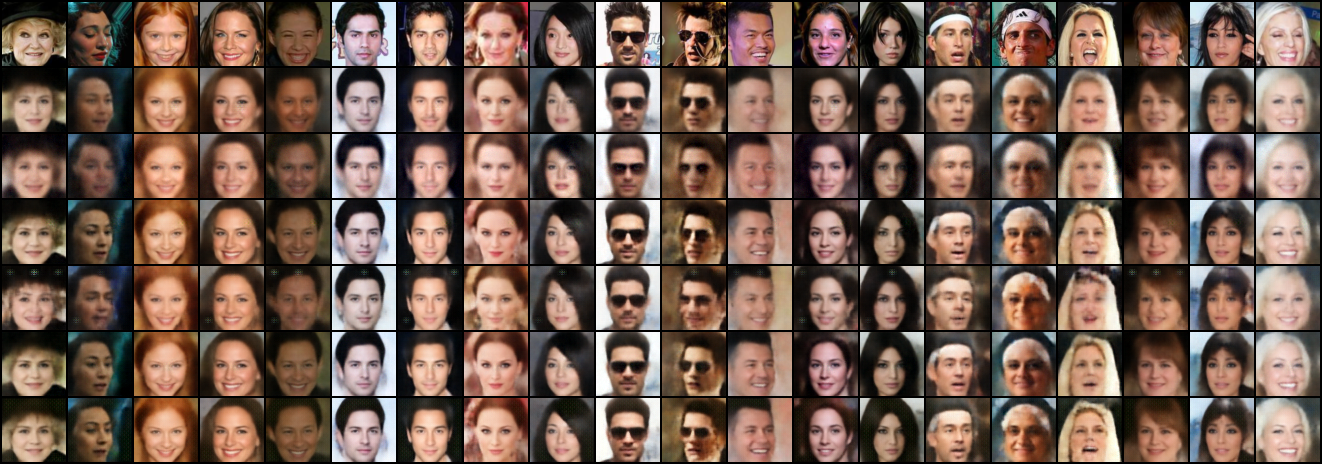}
    }
    \vspace{-10pt}
    \caption{Comparisons for various methods on the reconstruction quality of images.}
    \label{fig:cmp_rec}
\end{figure}
\begin{figure}[t]
    \centering
    \subfigure[VAE]{
    \includegraphics[width=0.22\textwidth]{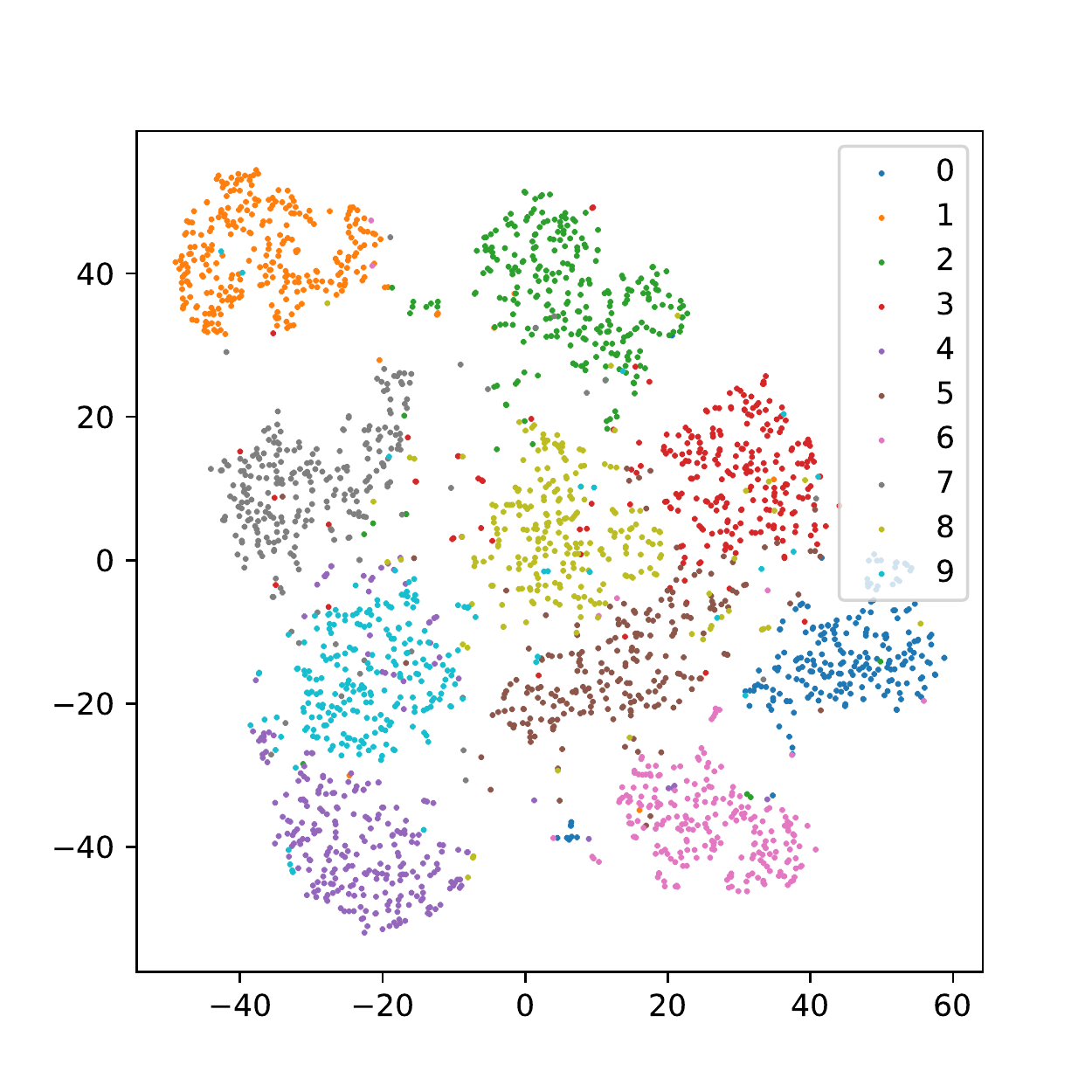}
    }
    \subfigure[VampPrior]{
    \includegraphics[width=0.22\textwidth]{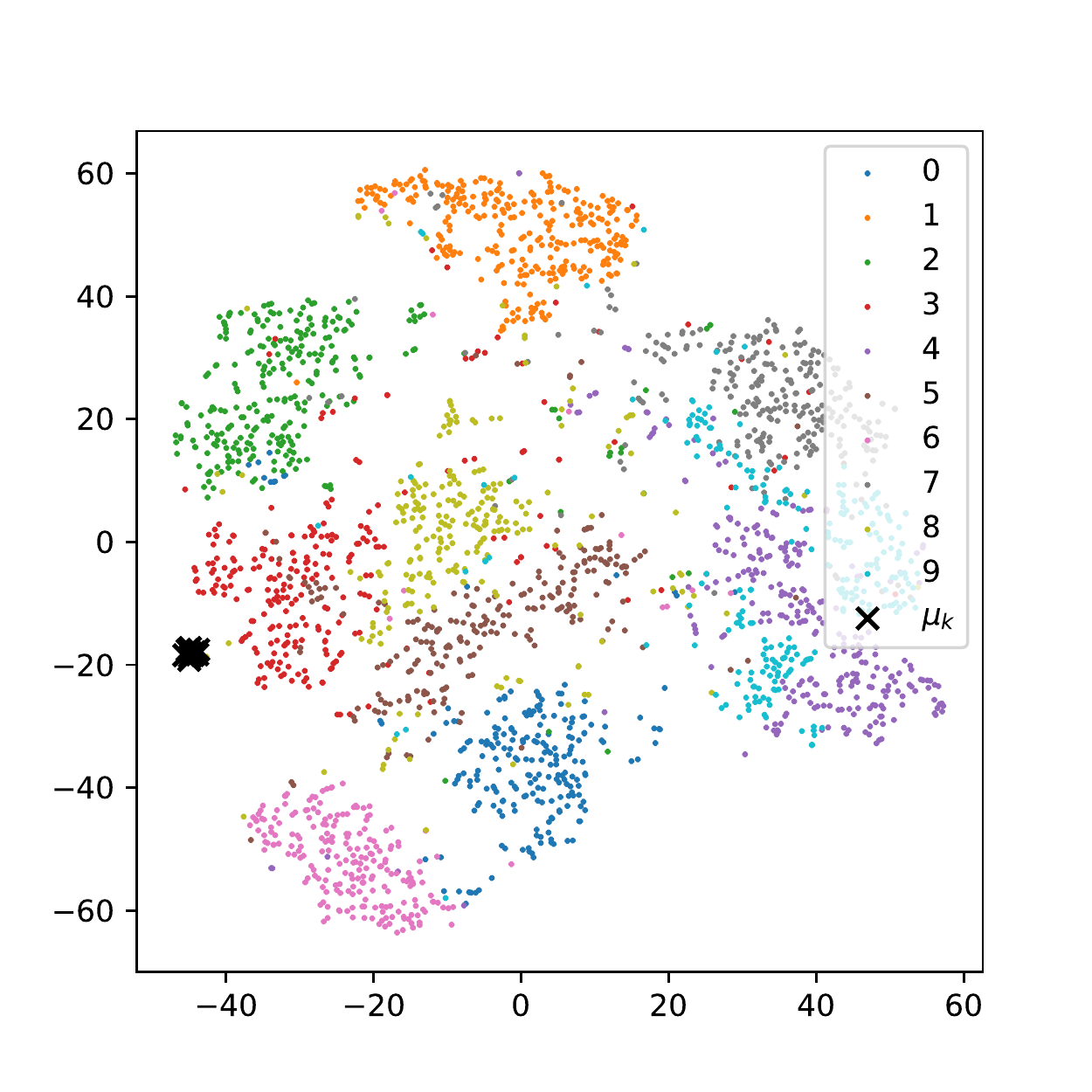}
    }
    \subfigure[GMVAE]{
    \includegraphics[width=0.22\textwidth]{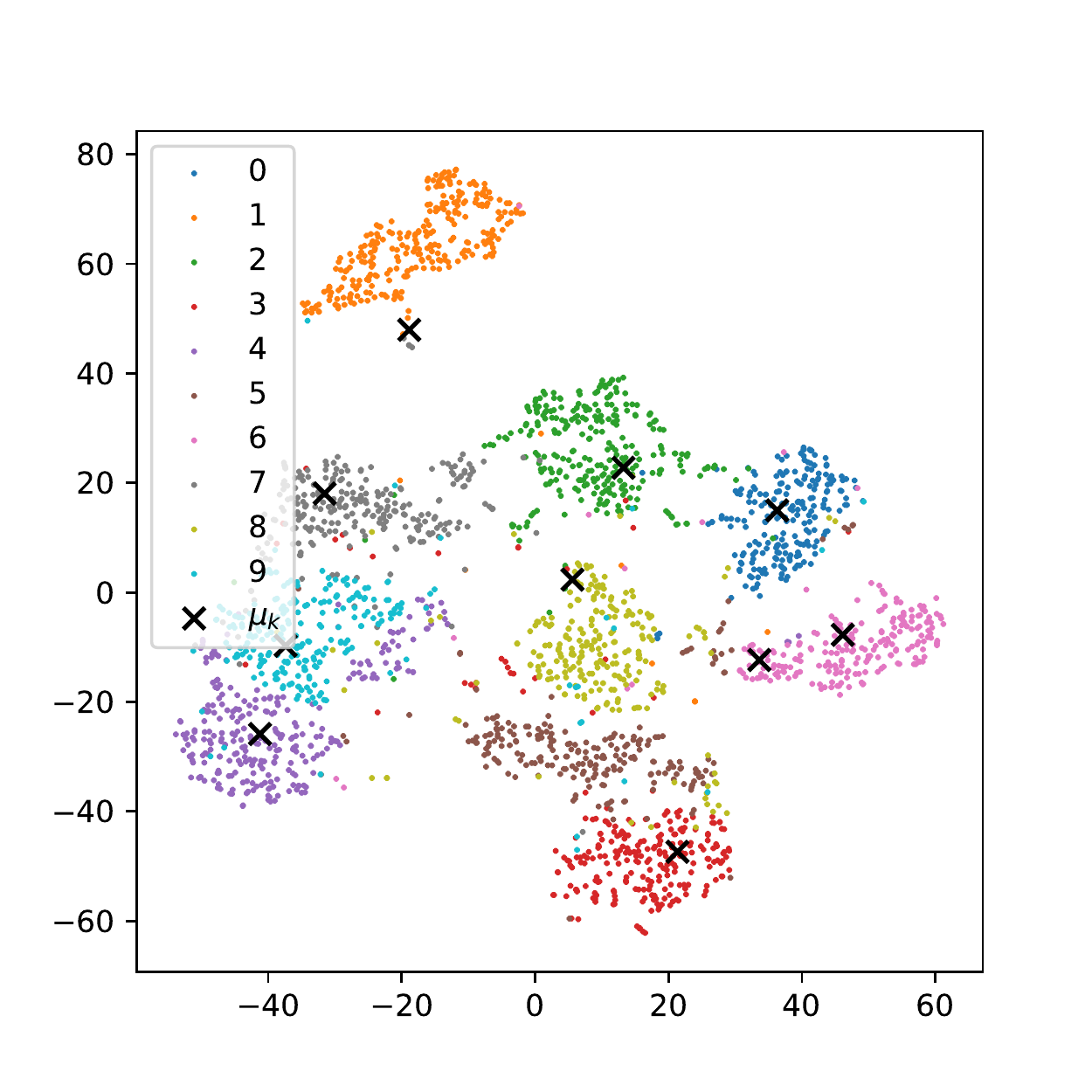}
    }
    \subfigure[WAE]{
    \includegraphics[width=0.22\textwidth]{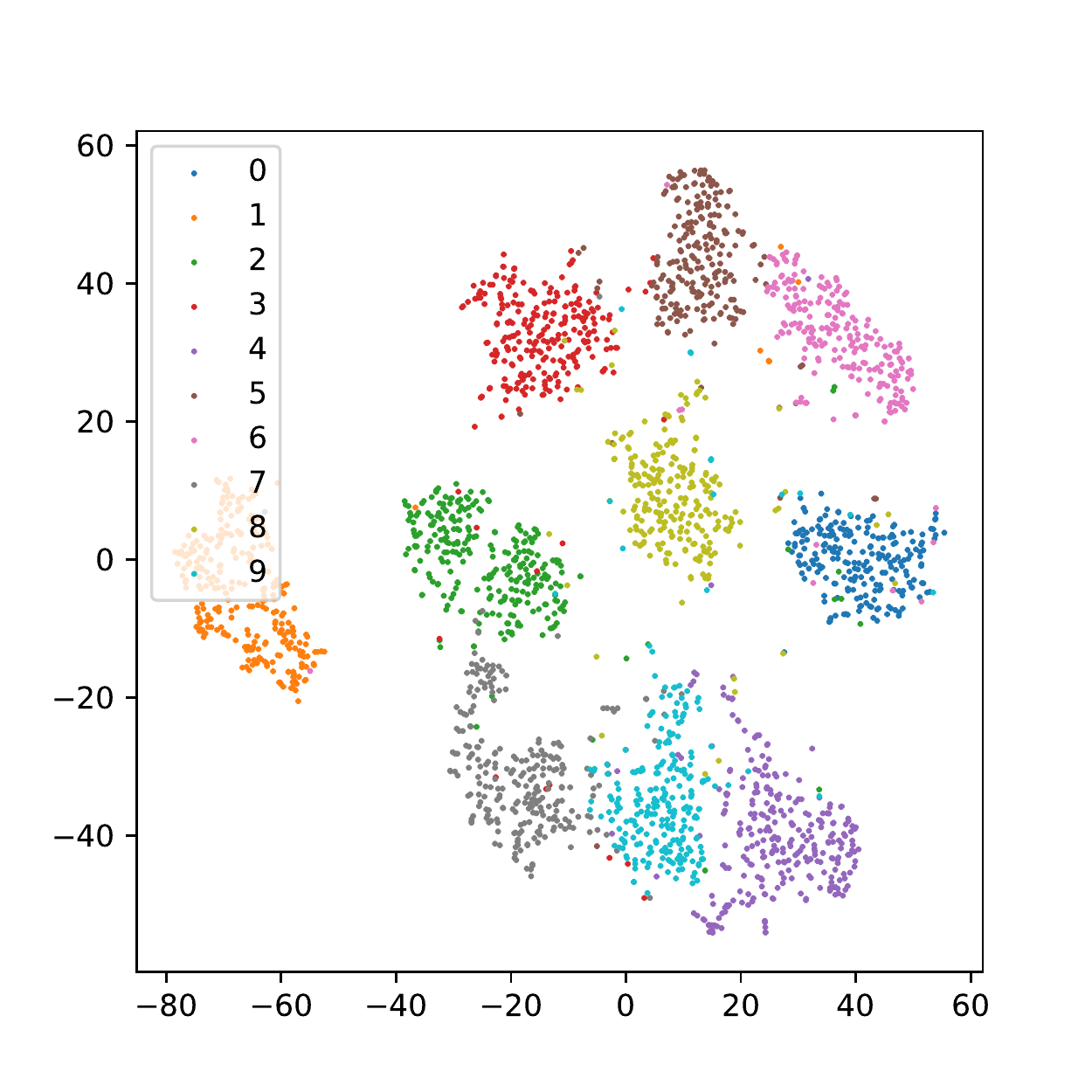}
    }
    \subfigure[SWAE]{
    \includegraphics[width=0.22\textwidth]{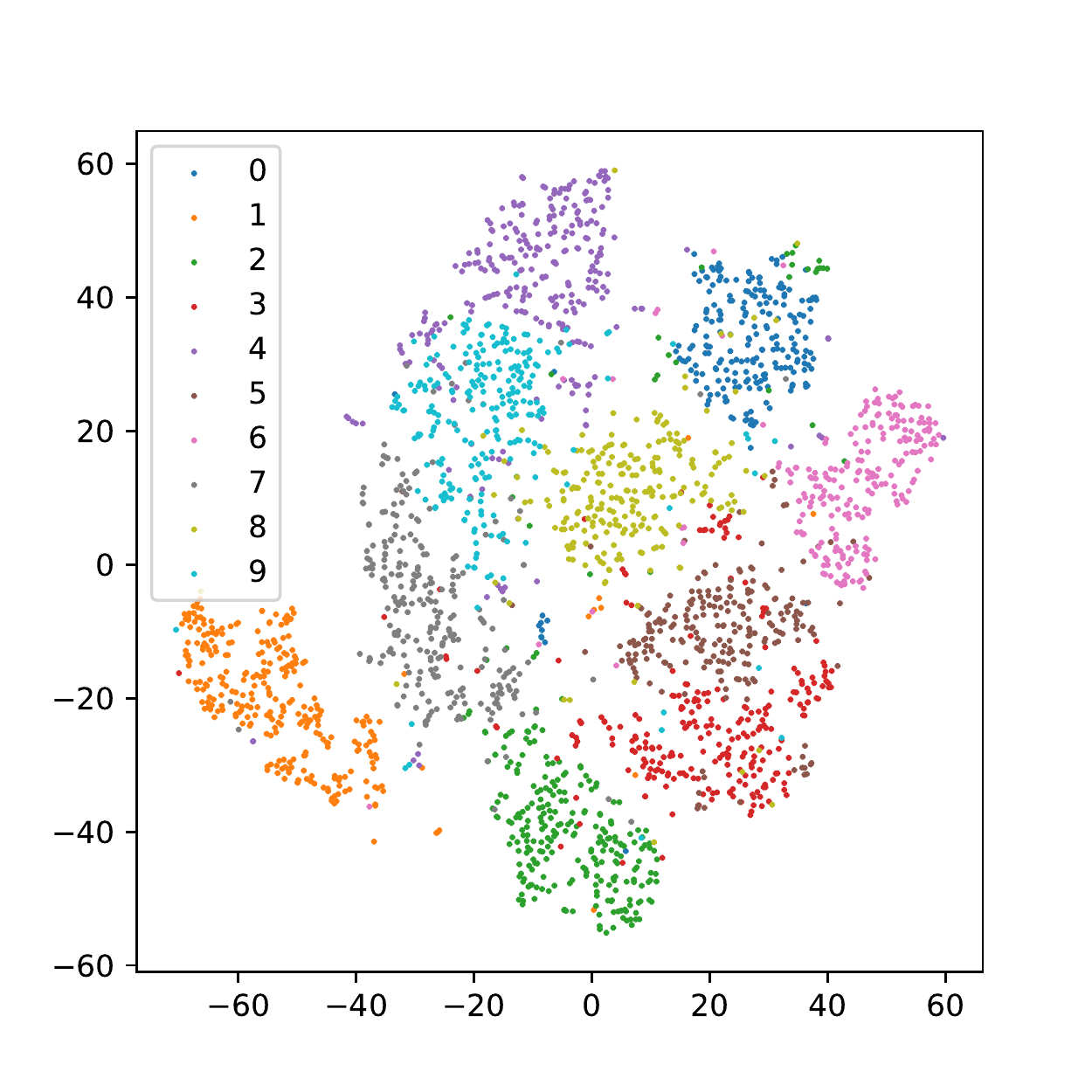}
    }
    \subfigure[Probabilistic RAE]{
    \includegraphics[width=0.22\textwidth]{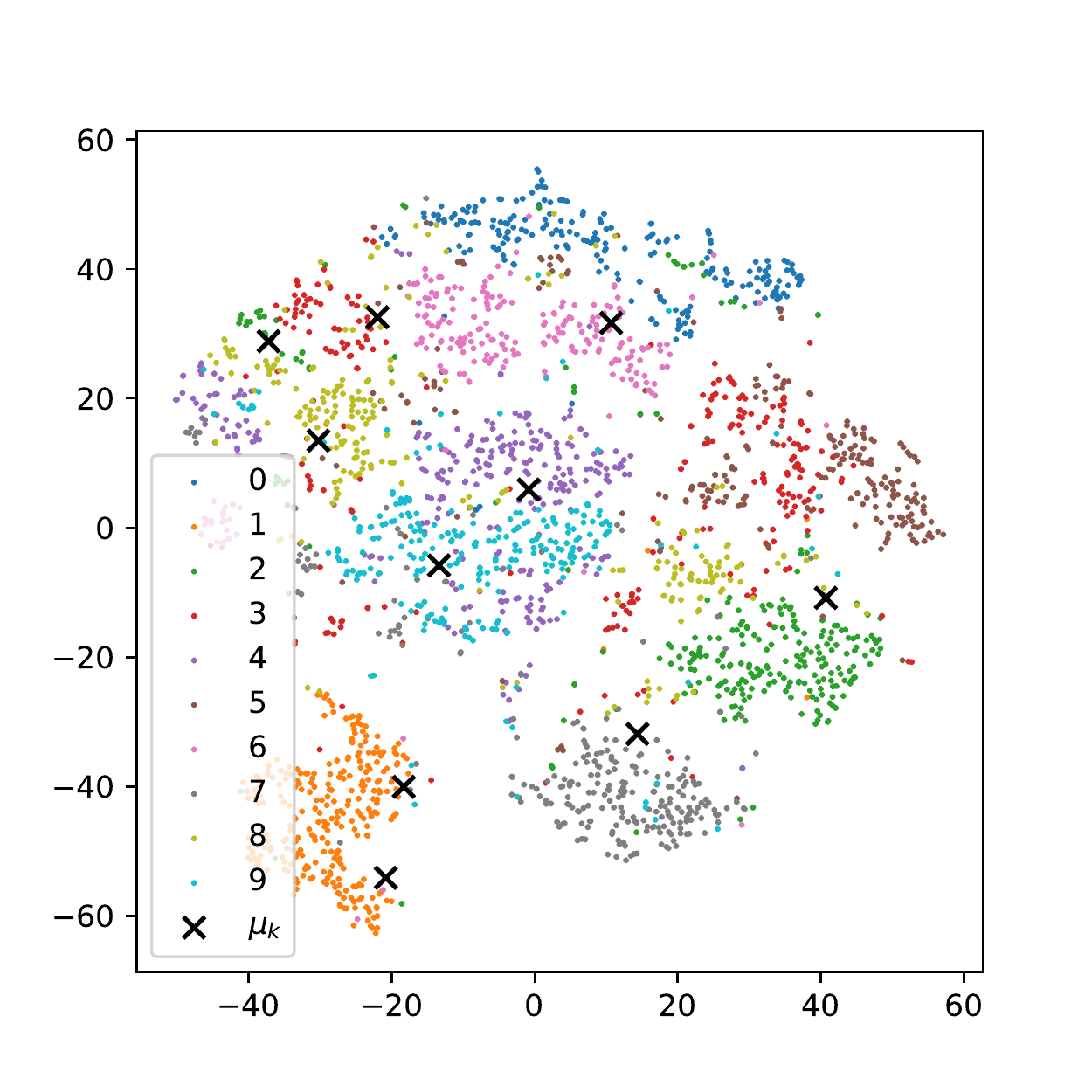}
    }
    \subfigure[Deterministic RAE]{
    \includegraphics[width=0.22\textwidth]{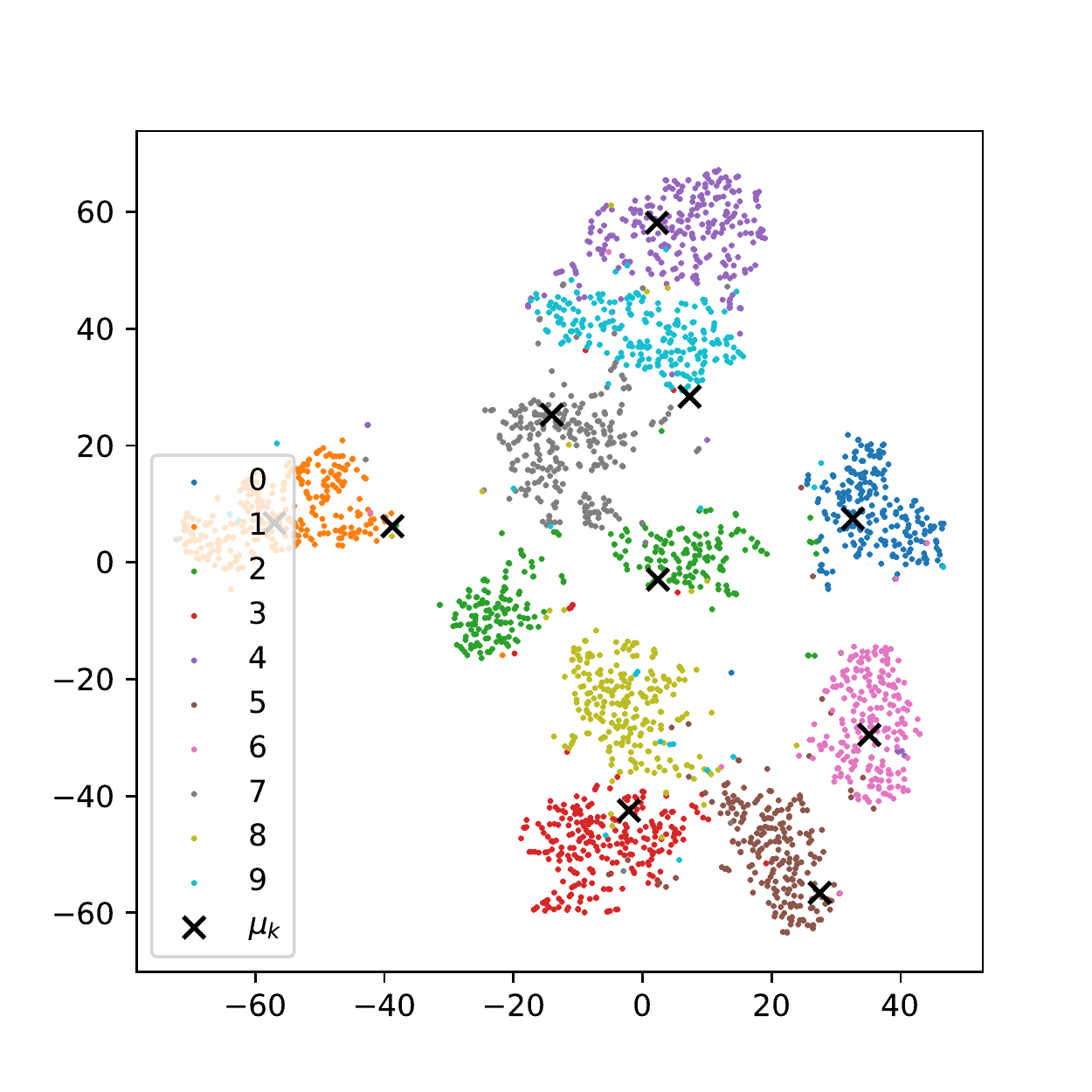}
    }
    \vspace{-10pt}
    \caption{Comparisons for various methods on the t-SNE embeddings of their latent codes of MNIST digits. 
    The embeddings corresponding to different digits are with different colors.    
    For the autoencoders using structured priors, the means of the Gaussian components are also embedded and shown as black crosses.}
    \label{fig:cmp_tsne}
\end{figure}

\section{More Experimental Results}\label{app:4}
\subsection{The influence of $\tau$ on our relational co-training}
For our relational co-training strategy, the hyperparameter $\tau$ controls the strength of the relational regularization for the posterior distributions. 
In our experiments, when $\tau=0$, (\ref{eq:cotrain}) degrades to learn two autoencoders independently. 
When $\tau=1$, on the contrary, (\ref{eq:cotrain}) ignores the constraint imposed by the predefined priors (Here, we assume the priors are normal distributions). 
In our experiments, we empirically set $\tau=0.5$ for all four multi-view learning datasets.
In Figure~\ref{eq:tau}, we further investigate the influence of $\tau$ on the learning results. 
According to the investigation results, we can find that setting $\tau\in [0.4, 0.8]$ achieves the best performance in most situtations.

\subsection{Visual comparisons on image generation}
Figure~\ref{fig:cmp_rec} compares various methods on the quality of reconstructed testing images. 
We can find that the reconstruction results obtained by our probabilistic RAE (P-RAE) and deterministic RAE (D-RAE) are at least comparable to their competitors. 

For the MNIST dataset, we derive the latent codes of the testing images for each autoencoder and compare different autoencoders on the t-SNE embeddings of their latent codes. 
Figure~\ref{fig:cmp_tsne} visualize these t-SNE embeddings. 
For the autoencoders with structured priors ($i.e.$, VampPrior, GMVAE, P-RAE, and D-RAE), the means of the Gaussian components in their priors are embeddings as well.
According to the t-SNE figures we can find that the VampPrior suffers from a serious mode collaspe problem when generating images because its Gaussian components are concentrated together, which can only represent limited modes.
On the contrary, the priors of GMVAE, P-RAE, and D-RAE have diverse Gaussian components, which cover the proposed latent space well.
For the CelebA dataset, we also compare various autoencoders on the t-SNE of their latent codes. 
In this case, we find that both GMVAE and P-RAE fail to learn structured priors --- their GMM-based priors are with poor diversity. 
However, compared with GMVAE, which generates an undesired modality, our P-RAE at least ensures that all the Gaussian components are valid and obey to the empirical distribution of the latent codes.
Our D-RAE achieves the best performance: the t-SNE results show that it successfully estimates a structured prior, whose Gaussian components are with good diversity and indicate different modalities in the latent space. 

Figures~\ref{fig:cmp_sample1} and~\ref{fig:cmp_sample2} give more randomly generated samples. 
The visual effects of these samples further demonstrate the superiority of our RAEs.

\begin{figure}[t]
    \centering
    \subfigure[VAE]{
    \includegraphics[width=0.3\textwidth]{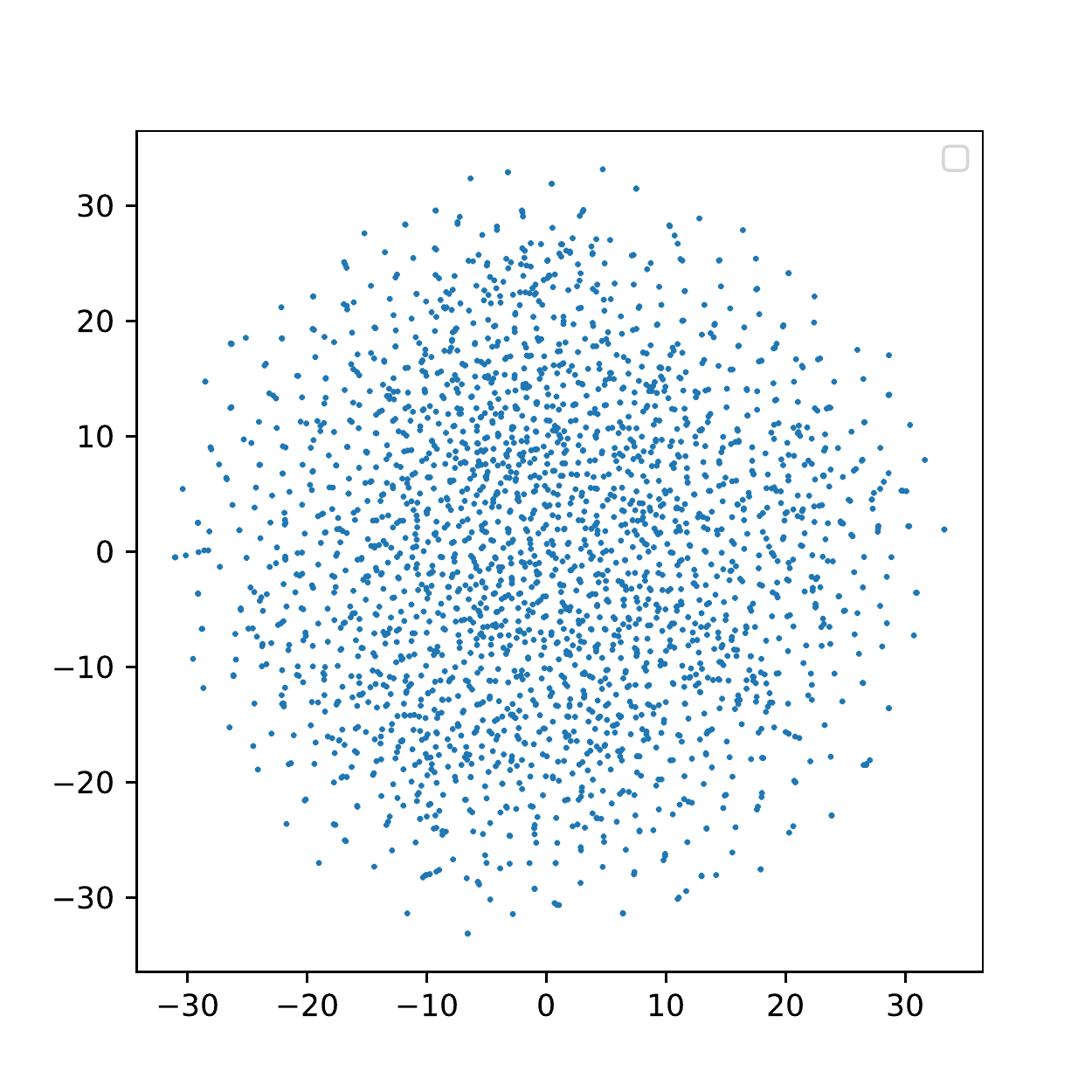}
    }
    \subfigure[GMVAE]{
    \includegraphics[width=0.3\textwidth]{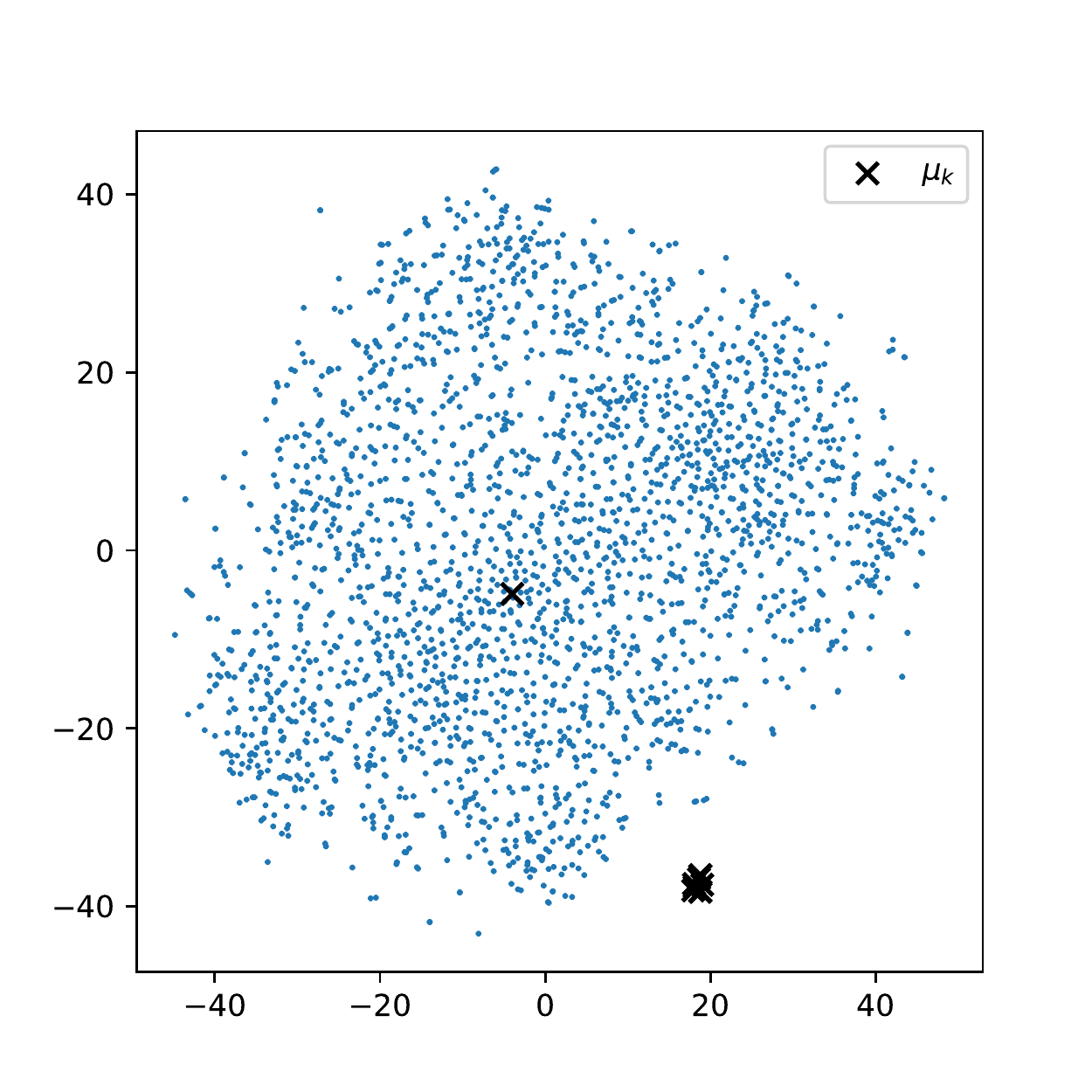}
    }
    \subfigure[WAE]{
    \includegraphics[width=0.3\textwidth]{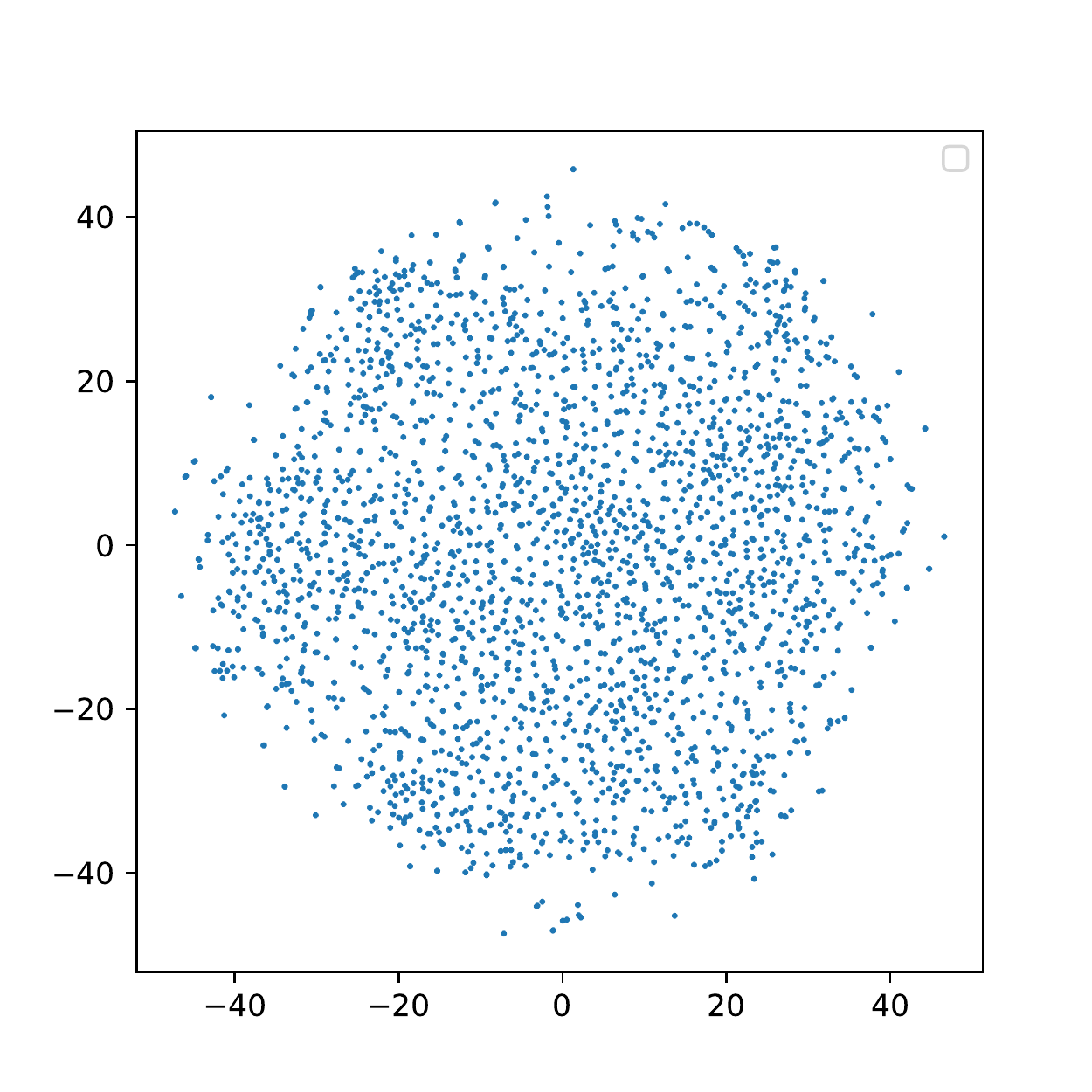}
    }
    \subfigure[SWAE]{
    \includegraphics[width=0.3\textwidth]{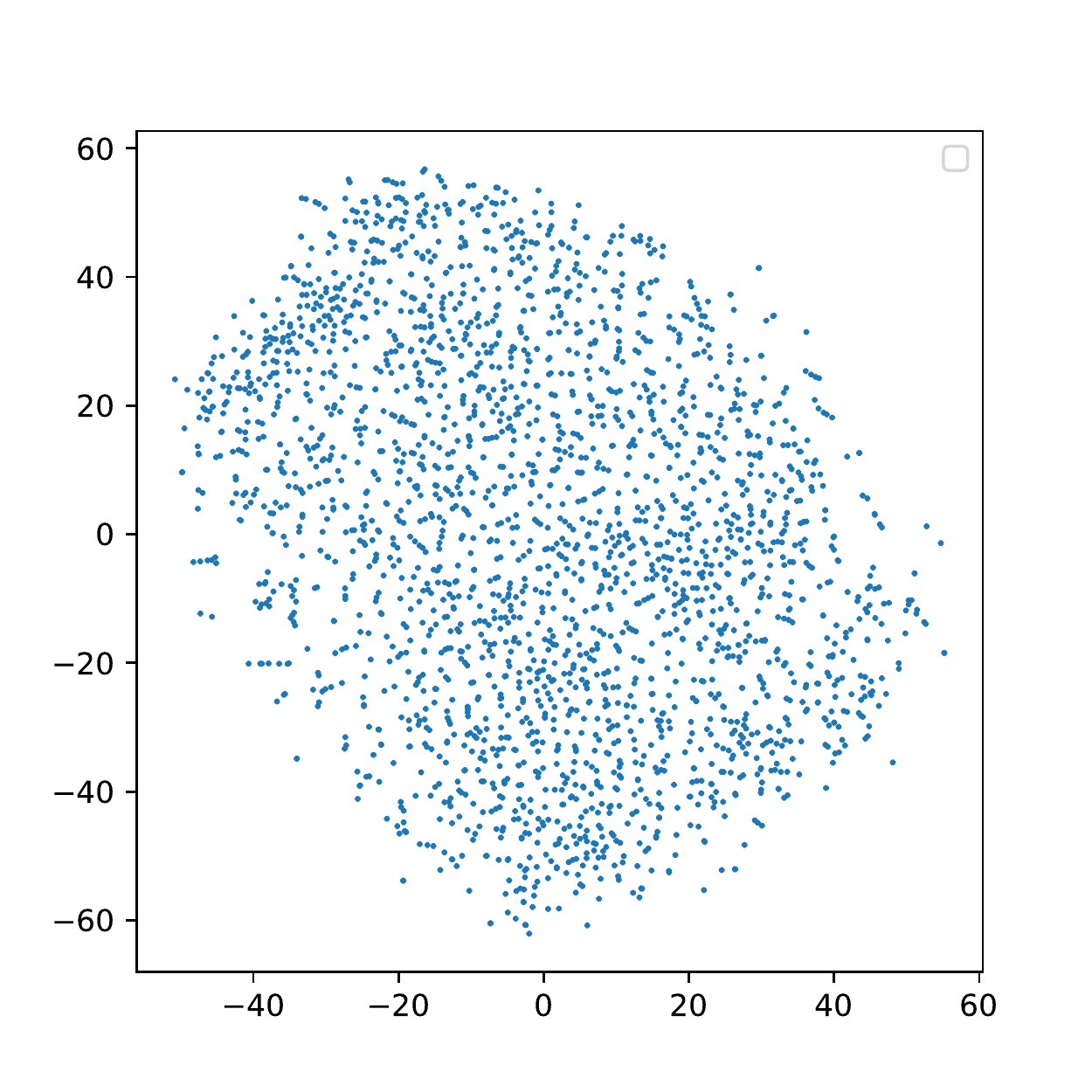}
    }
    \subfigure[Probabilistic RAE]{
    \includegraphics[width=0.3\textwidth]{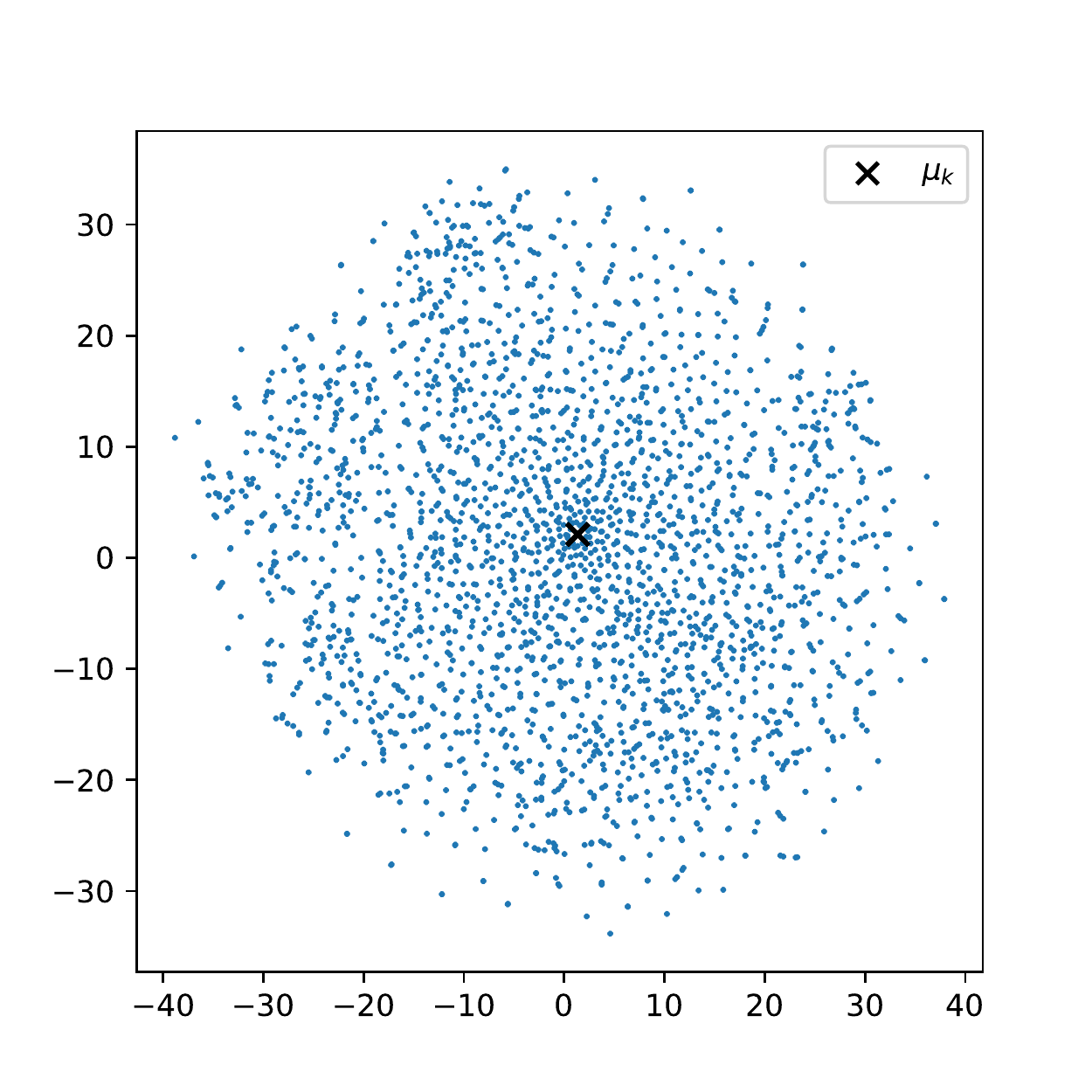}
    }
    \subfigure[Deterministic RAE]{
    \includegraphics[width=0.3\textwidth]{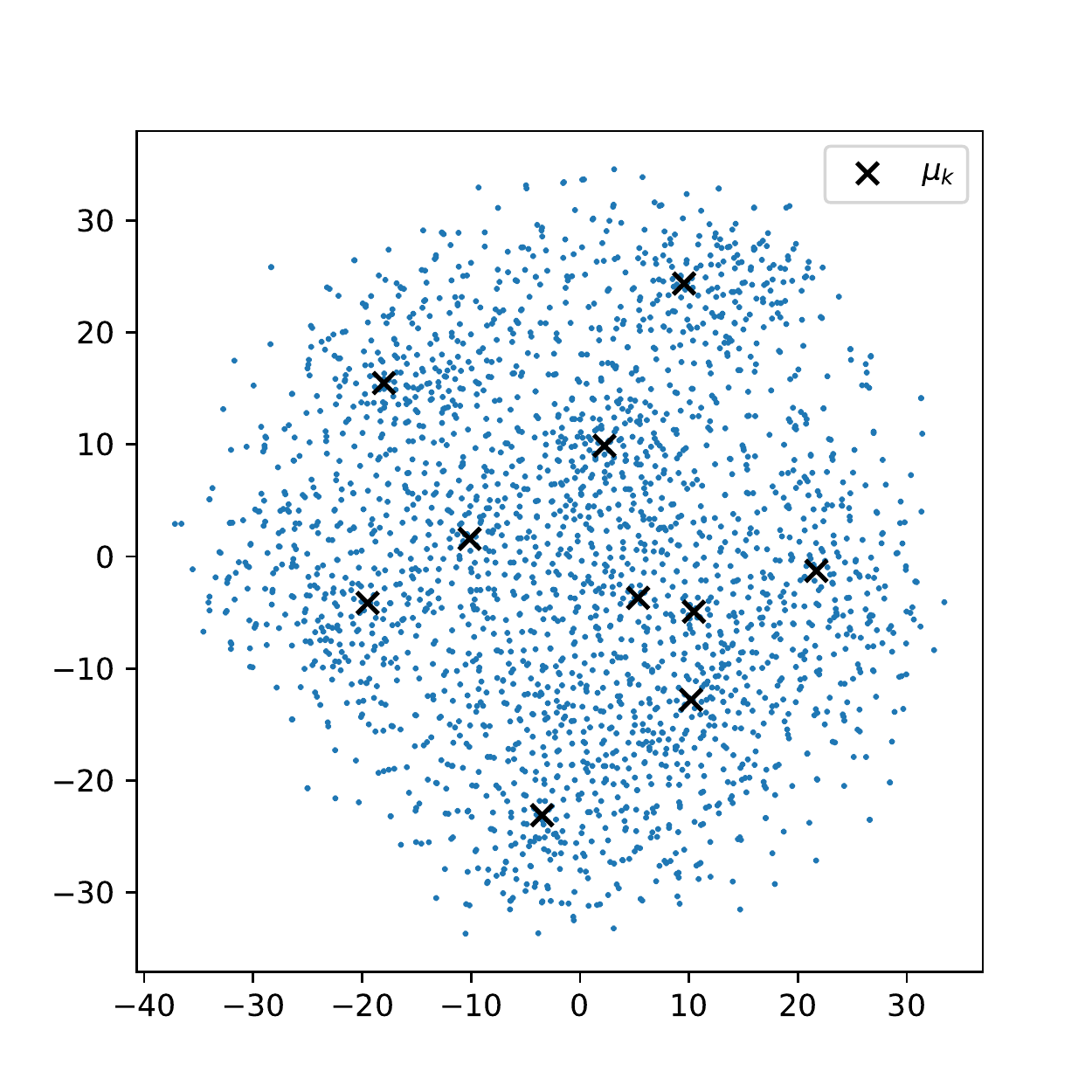}
    }
    \vspace{-10pt}
    \caption{Comparisons for various methods on the t-SNE embeddings of their latent codes of face images. 
    For the autoencoders using structured priors, the means of the Gaussian components are also embedded and shown as black crosses.}
    \label{fig:cmp_tsne2}
\end{figure}

\begin{figure}[t]
    \centering
    \subfigure[VAE]{
    \includegraphics[width=0.22\textwidth]{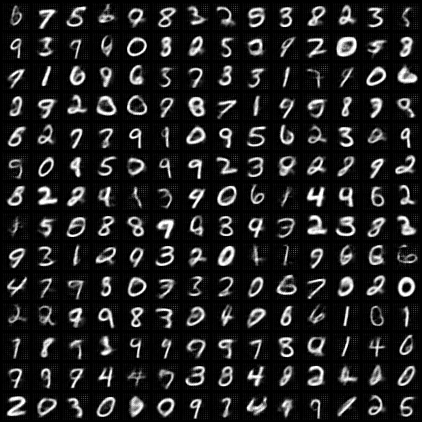}
    }
    \subfigure[VampPrior]{
    \includegraphics[width=0.22\textwidth]{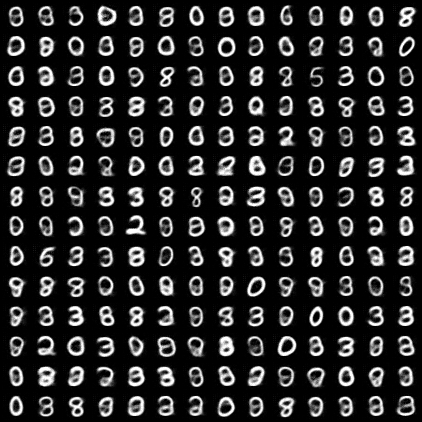}
    }
    \subfigure[GMVAE]{
    \includegraphics[width=0.22\textwidth]{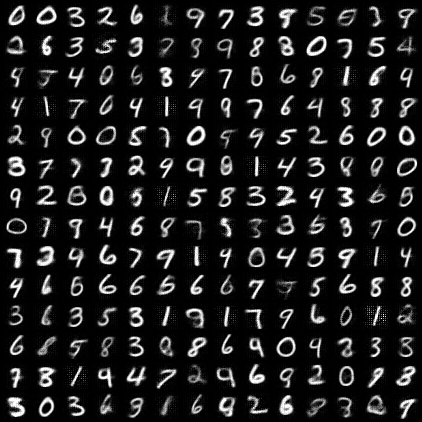}
    }
    \subfigure[WAE]{
    \includegraphics[width=0.22\textwidth]{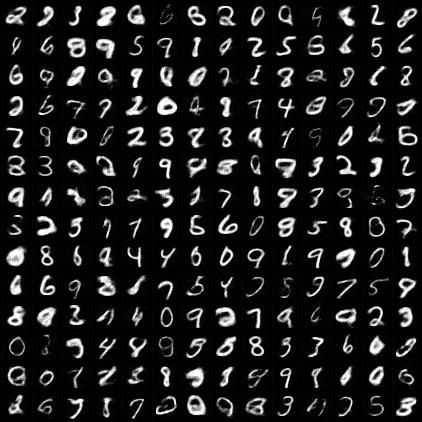}
    }
    \subfigure[SWAE]{
    \includegraphics[width=0.22\textwidth]{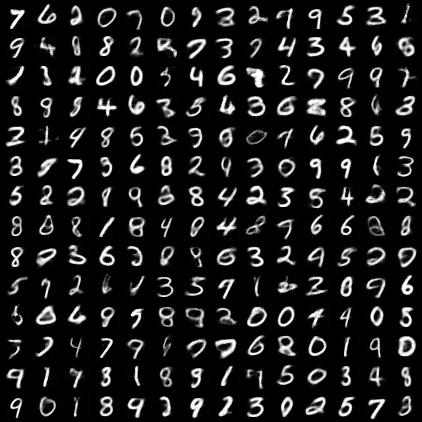}
    }
    \subfigure[Probabilistic RAE]{
    \includegraphics[width=0.22\textwidth]{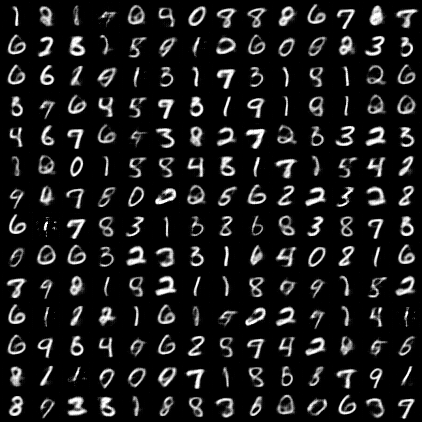}
    }
    \subfigure[Deterministic RAE]{
    \includegraphics[width=0.22\textwidth]{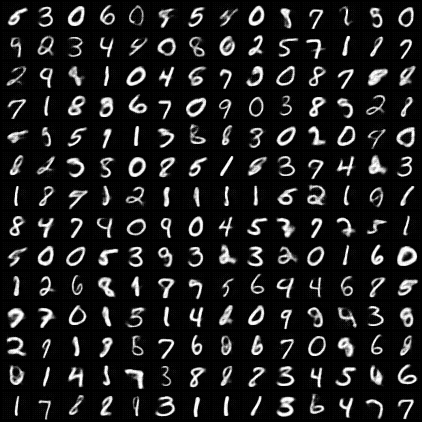}
    }
    \vspace{-10pt}
    \caption{Comparisons for various methods on digit generation.}
    \label{fig:cmp_sample1}
\end{figure}

\begin{figure}[t]
    \centering
    \subfigure[VAE]{
    \includegraphics[width=0.4\textwidth]{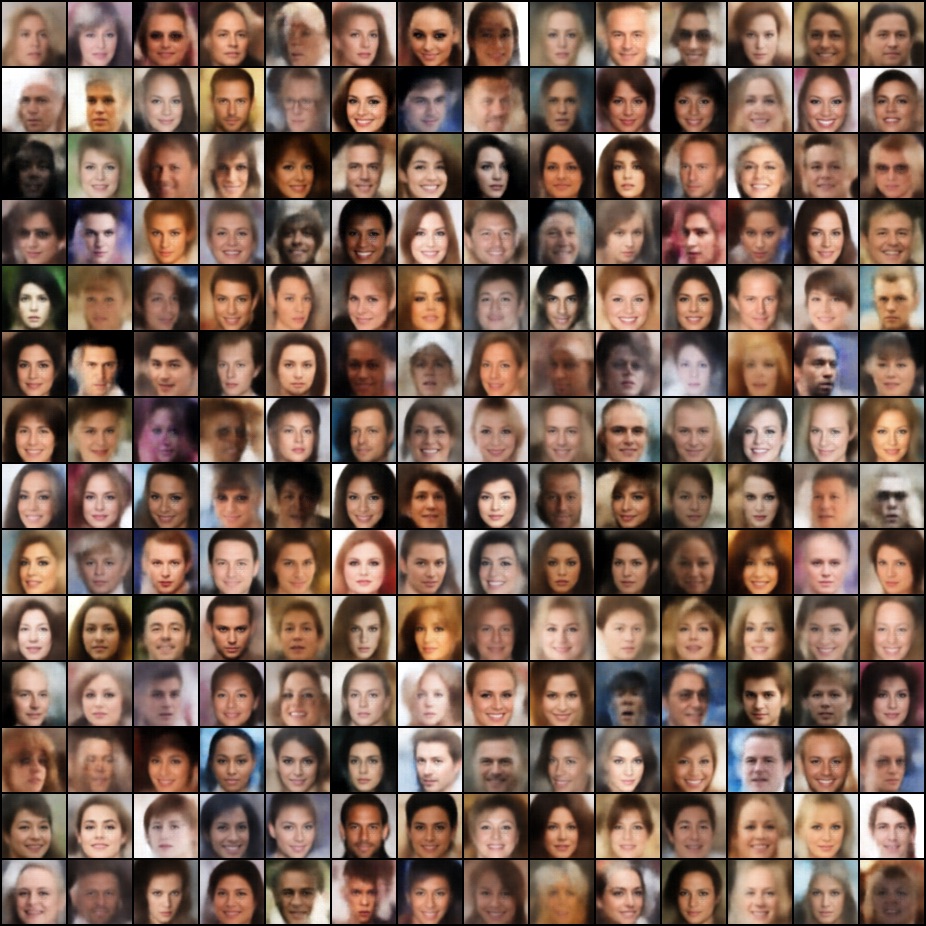}
    }
    \subfigure[GMVAE]{
    \includegraphics[width=0.4\textwidth]{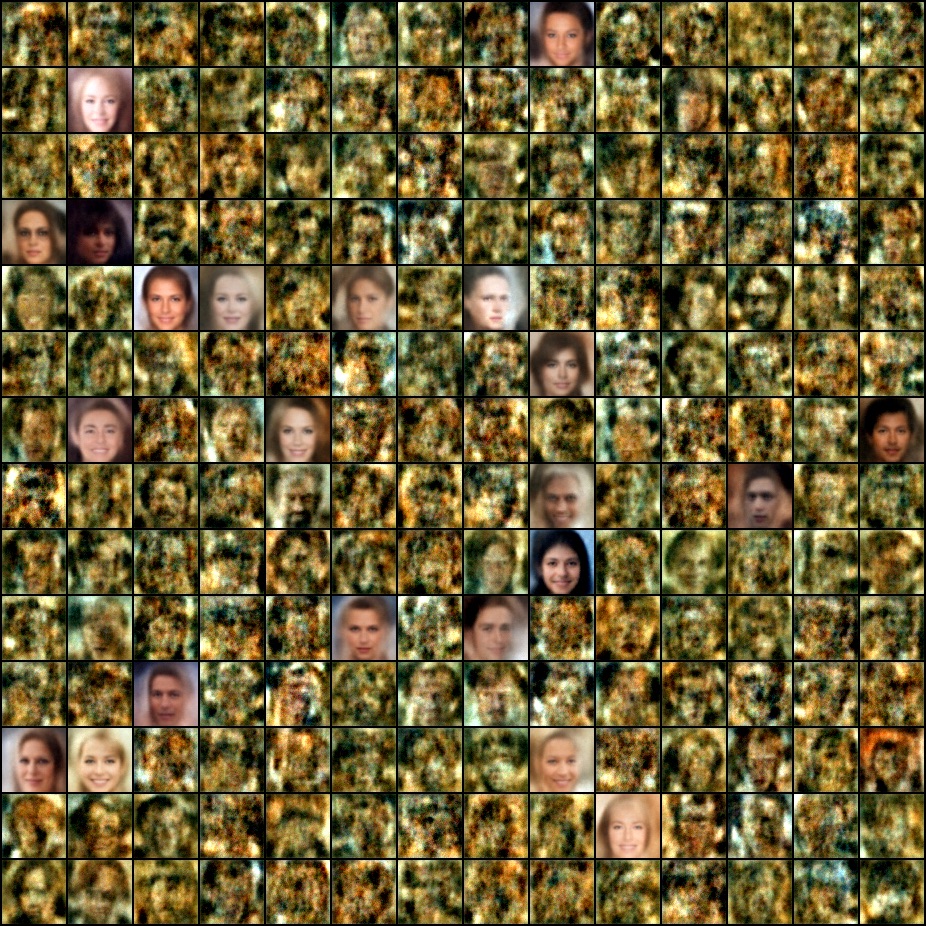}
    }
    \subfigure[WAE]{
    \includegraphics[width=0.4\textwidth]{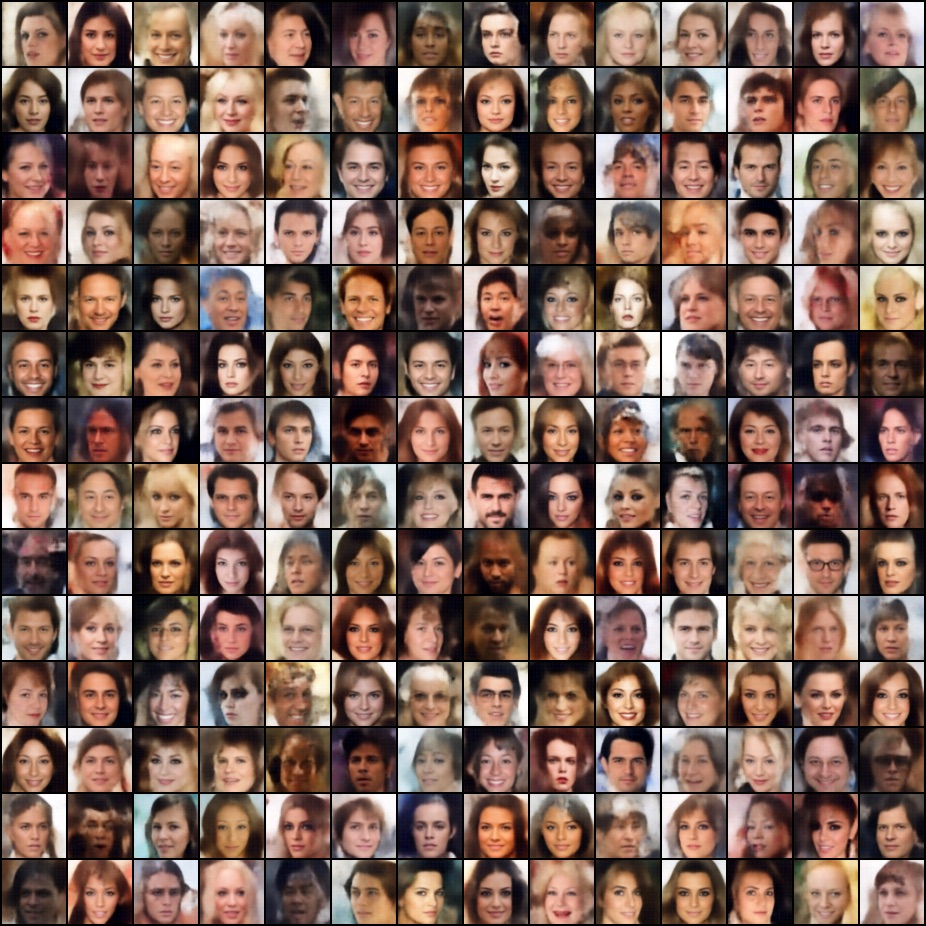}
    }
    \subfigure[SWAE]{
    \includegraphics[width=0.4\textwidth]{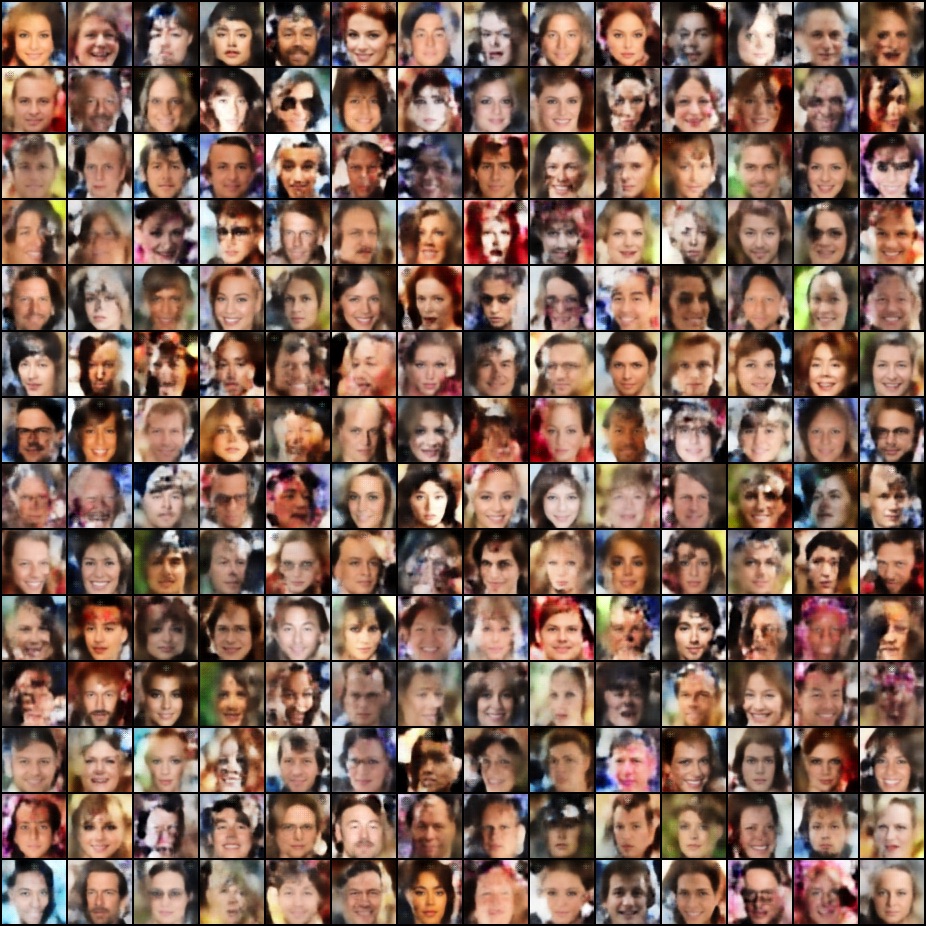}
    }
    \subfigure[Probabilistic RAE]{
    \includegraphics[width=0.4\textwidth]{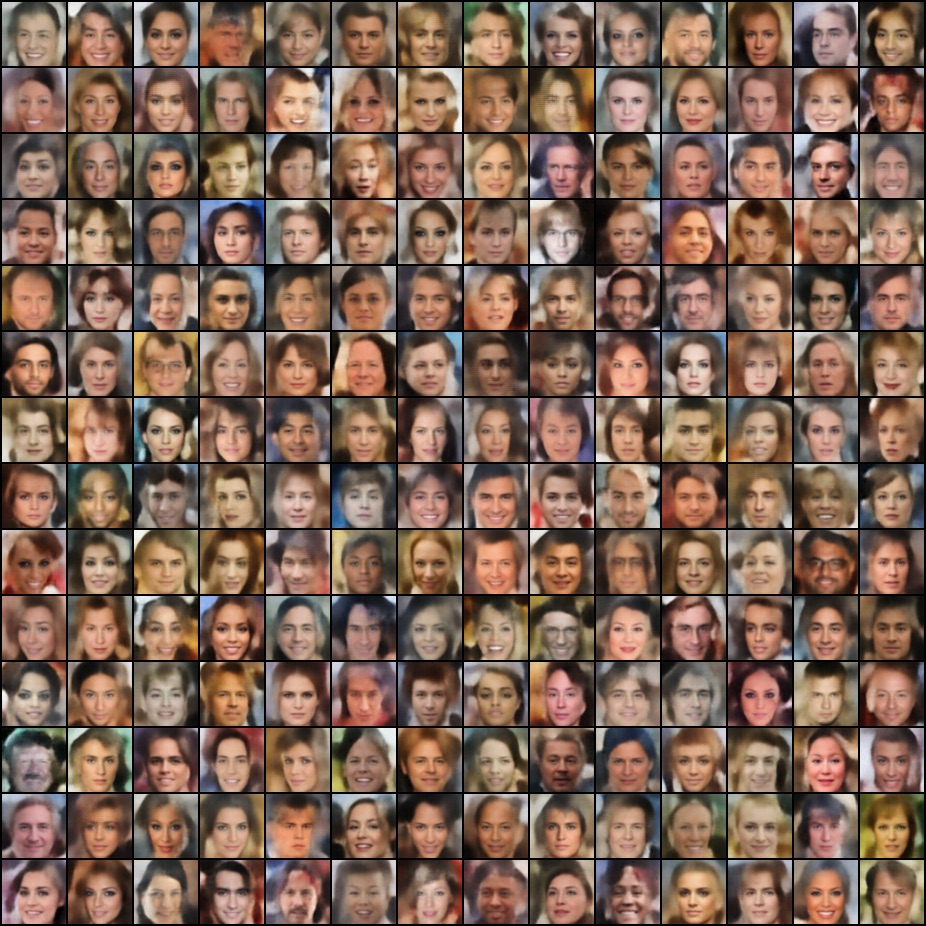}
    }
    \subfigure[Deterministic RAE]{
    \includegraphics[width=0.4\textwidth]{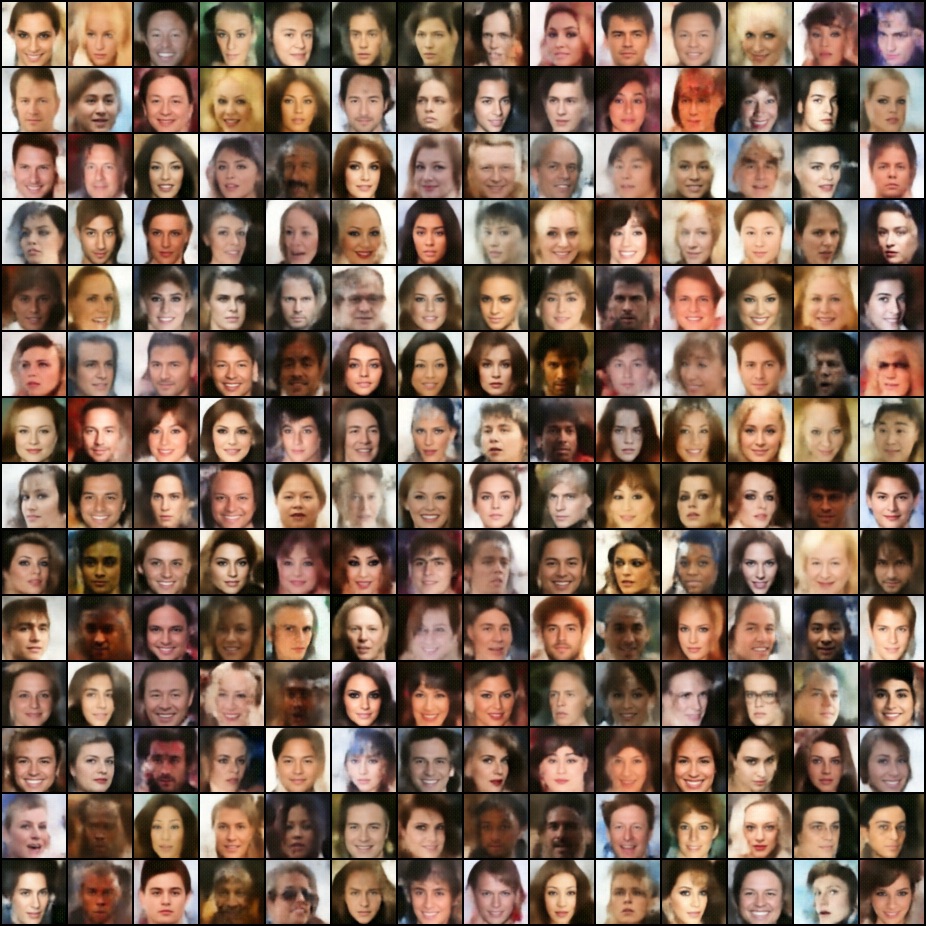}
    }
    \vspace{-10pt}
    \caption{Comparisons for various methods on face generation.}
    \label{fig:cmp_sample2}
\end{figure}

\end{document}